\documentclass[11pt]{article}

\usepackage{fullpage}
\usepackage{amsmath}
\usepackage{amssymb}
\usepackage{graphicx} 

\usepackage{mathrsfs}

\usepackage{fancybox,graphicx}
\usepackage[T1]{fontenc}
\usepackage[latin1]{inputenc}
\usepackage{float}

\usepackage{color}

\floatstyle{ruled}
\newfloat{alg}{tbp}{loa}
\floatname{alg}{Algorithm}

{
\rm
\begin{tabbing} 
....\=...\=...\=...\=...\=  \+ \kill
} %
{\end{tabbing}
}

\newtheorem{definition}{Definition}[section]
\newtheorem{lemma}{Lemma}[section]
\newtheorem{corollary}{Corollary}[section]
\newtheorem{theorem}{Theorem}[section]
\newtheorem{proposition}{Proposition}[section]

\newtheorem{remark}{Remark}[section]
\newtheorem{example}{Example}[section]

\newcommand{\G}{{\Gamma}}
\newcommand{\GS}{{\mathbf S}}

\newcommand{\bbeta}{{\bar{\beta}}}
\newcommand{\hbeta}{{\hat{\beta}}}
\newcommand{\tbeta}{{\tilde{\beta}}}

\newcommand{\rE}{{\mathbf E}}
\newcommand{\rP}{{\mathbf P}}
\newcommand{\sgn}{{\mathrm{sgn}}}

\newcommand{\tr}{{\mathrm{tr}}}

\newcommand{\spec}{{\mathrm{sp}}}
\newcommand{\supp}{{\mathrm{supp}}}
\newcommand{\Real}{{\mathbb{R}}}

\newcommand{\cor}{\mathrm{cor}}
\newcommand{\gw}{{\mathrm{width}}}

\newcommand{\BlackBox}{\rule{1.5ex}{1.5ex}}  
\newenvironment{proof}{\par\noindent{\bf Proof\ }}{\hfill\BlackBox\\[2mm]}

\newcommand{\cT}{{\cal T}}
\newcommand{\tcT}{{\tilde{\cal T}}}
\newcommand{\cC}{{\cal C}}

\newcommand{\innerprod}[2]{{\langle {#1} , {#2} \rangle}}

\newcommand{\bel}{\begin{eqnarray}\label}
\newcommand{\eel}{\end{eqnarray}}
\newcommand{\bes}{\begin{eqnarray*}}
\newcommand{\ees}{\end{eqnarray*}}
\newcommand{\obeta}{\beta_*}

\def\lam{\lambda}
\def\scrA{{\mathscr A}}
\def\scrB{{\mathscr B}}
\def\scrAD{{{\mathscr A},D}}
\def\scrBD{{{\mathscr B},D}}

\def\scrM{{\mathscr M}}
\def\calT{{\cal T}}
\def\ker{\hbox{\rm ker}}

\def\argmin{\mathop{\rm arg\, min}}
\def\pa{\partial}

\def\Omegabar{{\bar\Omega}}

\begin{document}

\title{A General Framework of Dual Certificate Analysis for Structured Sparse Recovery Problems}
\date{}
\author{
Cun-Hui Zhang\thanks{Research partially supported by the
NSF Grants DMS 0804626, DMS 0906420, 
and NSA Grant H98230-11-1-0205} \\
Department of Statistics \\
Rutgers University, NJ \\
{\it czhang@stat.rutgers.edu} \\
\and
Tong Zhang\thanks{Research partially supported by the following grants: AFOSR-10097389, NSA
-AMS 081024, NSF DMS-1007527, and NSF IIS-1016061} \\
Department of Statistics\\ 
Rutgers University, NJ\\
{\it tzhang@stat.rutgers.edu}
}

\maketitle

\begin{abstract}
  This paper develops a general theoretical framework to analyze structured sparse recovery problems
  using the notation of {\em dual certificate}.
  Although certain aspects of the dual certificate idea have already been used in some 
  previous work,  due to the lack of a general and coherent theory,
  the analysis has so far only been carried out in limited scopes for specific problems. 
  In this context the current paper makes two contributions.
  First, we introduce a general definition of dual certificate,
  which we then use to develop a unified theory of sparse recovery analysis for convex programming.
  Second, we present a class of structured sparsity regularization called
  {\em structured Lasso} for which 
  calculations can be readily performed under our theoretical framework.
  This new theory includes many seemingly loosely 
  related previous work as special cases; it also implies new results 
  that improve existing ones even for standard formulations such as $\ell_1$ regularization.
\end{abstract}

\section{Introduction}

This paper studies a general form of the sparse recovery problem, where our goal is to estimate
a certain signal $\bbeta_*$ from observations. 
We are especially interested in solving this problem using convex programming; that is,
given a convex set $\Omega$,
our estimator $\hbeta$ is obtained from the following regularized minimization problem:
\begin{equation}
\hbeta= \arg\min_{\beta \in \Omega} \left[ L(\beta) + R(\beta) \right] . \label{eq:hbeta}
\end{equation}
Here $L(\beta)$ is a loss function, which measures how closely $\beta$ matches the observation;
and $R(\beta)$ is a regularizer, which captures the structure of $\bbeta_*$.
Note that the theory developed in this paper does not need to assume that 
$\bbeta_* \in \Omega$ although this is certainly a desirable property (especially if we would like to recover 
$\bbeta_*$ without error).
Our primary interest is in the case where $\Omega$ lives in an Euclidean space $\Omegabar$. 
However, our analysis holds automatically when $\Omega$ is contained in a separable Banach 
space $\Omegabar$, and both $L(\cdot)$ and $R(\cdot)$ are convex functions that are 
defined in the whole space $\Omegabar$, both inside and outside of $\Omega$. 

As an example, assume that $\bbeta_*$ is a 
$p$ dimensional vector: $\bbeta_* \in \Real^p$; we observe a vector 
$y \in \Real^n$ and an $n \times p$ matrix $X$ such that 
\[
y=X \bbeta_* + \text{noise}. 
\]
We are interested in estimating $\bbeta_*$ from the noisy
observation $y$. However, in modern applications we are mainly interested in
the high dimensional situation where $p \gg n$. 
Since there are more variables than the number of observations, traditional statistical methods such as 
least squares regression will suffer from the so-called curse-of-dimensionality problem.
To remedy the problem, it is necessary to impose structures on $\bbeta_*$; and a popular assumption is
sparsity. That is $\|\bbeta_*\|_0=|\supp(\bbeta_*)|$ is smaller than $n$, where
$\supp(\beta) = \{j: \beta_j \neq 0\}$. 
A direct formulation of sparsity constraint leads to the nonconvex $\ell_0$ regularization formulation,
which is difficult to solve. A frequent remedy is to employ the so-called {\em convex relaxation}
approach,
where the $\ell_0$ regularization is replaced by an $\ell_1$ regularizer $R(\beta)=\lambda \|\beta\|_1$
that is convex. If we further consider the least squares loss $L(\beta)=\|y - X\beta\|_2^2$, then
we obtain the following $\ell_1$ regularization method (Lasso)
\begin{equation}
  \hbeta=\arg\min_{\beta \in \Real^p} \left[ \|y-X\beta\|_2^2 + \lambda \|\beta\|_1 \right] , \label{eq:L1}
\end{equation}
where $\Omega$ is chosen to be the whole parameter space $\Omegabar=\Real^p$.

\section{Related Work}

In sparse recovery analysis, we want to know how good is our estimator $\hbeta$ in comparison
to the target $\bbeta_*$. Consider the standard $\ell_1$ regularization method (\ref{eq:L1}),
two types of theoretical questions are of interests. 
The first is support recovery; that is, whether $\supp(\hbeta) = \supp(\bbeta_*)$. 
The second is parameter estimation; that is, how small is $\|\hbeta-\bbeta_*\|_2^2$.
The support recovery problem is often studied under the so-called {\em irrepresentable condition}
(some types also referred more generally as coherence condition) 
\cite{MeinshausenB06,Tropp06,ZhaoYu06,Wainwright09},
while the parameter estimation problem is often studied under the so-called {\em restricted isometry property} (or RIP) as well as its generalizations \cite{CandesTao07,ZhangHuang08,BickelRT09,ZhangT09,vandeGeerB09,YeZ10}.
Related ideas have been extended to more complex structured sparse regularization problems
such as group sparsity \cite{HuangZhang09,LMTG09} and certain matrix problems \cite{KoTsLo10,NegaWain10,KoltchinskiiLT11}. 
Closely related to parameter estimation is the so-called {\em oracle inequality}, which is particularly suitable for the dual-certificate
analysis considered here.

This paper is interested in the second question of parameter estimation, and the related problem of sparse oracle inequality.
Our goal is to present a general theoretical framework
using the notation of dual certificate to analyze sparse regularization problems such as the standard
Lasso (\ref{eq:L1}) as well as its generalization to more complex structured sparsity problems in (\ref{eq:hbeta}).
We note that there were already some recent attempts in developing such a general theory 
such as \cite{NeRaWaYu10} and \cite{ChRePaWi10}, but both have limitations. 
In particular the technique of  \cite{ChRePaWi10} only applies to noise-less regression problems with Gaussian random design
(its main contribution is the nice observation that Gordon's minimum singular value result can be applied to structured sparse recovery problems; the consequences will be further investigated in our paper); results in
\cite{ChRePaWi10} are subsumed by our more general results given in Section~\ref{sec:gordon}.
The analysis in \cite{NeRaWaYu10} relied on a direct generalization of RIP for decomposable regularizers which has technical limitations in its applications to more complex structured problems such as matrix regularization:  
the technique of RIP-like analysis and its generalization such as \cite{KoTsLo10,NegaWain10} gives performance bounds that do not imply exact recovery even when the noise is zero, while the technique we investigate here (via the notation of dual certificate) can get exact recovery \cite{CandesT09,Recht09}. In addition, not all regularizers can be easily considered as decomposable (for example, the mixed norm example in Section~\ref{sec:mixed-norm} is not). Even for Gaussian random design, the complexity statement in 
Section~\ref{sec:gordon} replies only on Gaussian width calculation that is more general than decomposable.
Therefore our analysis in this paper extends those of \cite{NeRaWaYu10} in multiple ways.

While the notation of dual certificate has been successfully employed 
in some earlier work (especially for some matrix regularization problems) such as 
\cite{Recht09,CaLiMaWr11,HsKaTz11-robust}, these results focused on special problems without
a general theory. In fact, from earlier work it is not even clear what should be a 
general definition of dual certificate for structured sparsity formulation (\ref{eq:hbeta}).
This paper addresses this issue. Specifically we will provide a general definition of dual certificate
for the regularized estimation problem (\ref{eq:hbeta}) and demonstrate that this definition 
can be used to develop a theoretical framework to analyze the sparse recovery performance of $\hbeta$
with noise. Not only does it provide a direct generalization of earlier work such as 
\cite{Recht09,CaLiMaWr11,HsKaTz11-robust}, but also it unifies RIP type analysis (or its generalization to restricted strong convexity) such as \cite{CandesTao07,NeRaWaYu10} and irrepresentable (or incoherence) conditions such as \cite{ZhaoYu06,Wainwright09}.
In this regard the general theory also includes as special cases some recent work by Candes and Plan that tried to develop non-RIP analysis for $\ell_1$ regularization \cite{CandesPlan09,CandesPlan11}.  In fact, even for the simple case of $\ell_1$ regularization, we show that our theory can lead to new and sharper results than existing ones. 

Finally, we would like to point out that while this paper successfully unifies the irrepresentable (or incoherence) conditions and RIP conditions under the general method of dual certificate, our analysis does not subsume some of the more elaborated analysis such as \cite{ZhangT09} and \cite{YeZ10} as special case. Those studies employed a different generalization of RIP which we may refer to as the {\em invertibility factor} approach using 
the terminology of  \cite{YeZ10}. It thus remains open whether it is possible to develop an even more general theory that can include all previous sparse recovery analysis as special cases.

\section{Primal-Dual Certificate}

As mentioned before, while fragments of the dual certificate idea has appeared before, 
there are so far no general definition and theory. Therefore in this section we will introduce 
a formal definition that can be used to analyze (\ref{eq:hbeta}). 
Recall that the parameter space $\Omega$ lives in a separable Banach space $\Omegabar$. 
Let $\Omegabar^*$ be the dual Banach space of $\Omegabar$ containing all continuous 
linear functions $u(\beta)$ defined on $\Omegabar$. We use $\innerprod{u}{\beta} = u(\beta)$ 
to denote the bi-linear function defined on $\Omegabar^*\times \Omegabar$. 
If $\Omegabar$ is an Euclidean space, then $\innerprod{\cdot}{\cdot}$ is just an 
inner product. In this notation $\innerprod{\cdot}{\cdot}$, the first argument is always in 
the dual space $\Omegabar^*$ and the second in the primal space $\Omegabar$. 
This allows as to keep track of the geometrical interpretation of our analysis even when 
$\Omegabar$ is an Euclidean or Hilbert space with $\Omegabar^*=\Omegabar$. 
In what follows, we will endow $\Omegabar^*$ with the weak topology: $u_k\to u$ iff 
$\innerprod{u_k-u}{\beta}\to 0$ for all $\beta\in\Omegabar$. This is equivalent to 
$\|u_k-u\|_D\to 0$ for any norm $\|\cdot\|_D$ in $\Omegabar^*$ when $\Omegabar$ is 
an Euclidean space. 

In the following, given any convex function $\phi(\cdot)$, 
we use the notation $\nabla \phi(\beta)\in\Omega^*$ to
denote a subgradient of $\phi(\beta)$ with respect to the geometry of $\Omegabar$ 
in the following sense:
\[
\phi(\beta') \geq \phi(\beta) + \innerprod{\nabla \phi(\beta)}{\beta'-\beta},\ \forall\ \beta'. 
\]
By convention, we also use
$\partial \phi(\beta)$ to denote its sub-differential (or the set of subgradient at $\beta$). 
The sub-differential is always a closed convex set in $\Omegabar^*$. 
Moreover, we define the Bregman divergence with respect to $\phi$ as:
\[
D_\phi(\beta,\beta')=\phi(\beta)-\phi(\beta')- \innerprod{\nabla \phi(\beta')}{\beta-\beta'} .
\]
Clearly, by the definition of sub-gradient, Bregman divergence is non-negative. 
These quantities are standard in convex analysis; for example, 
additional details can be found in \cite{Roc70}.

Instead of working directly with the target $\bbeta_*$, we consider an approximation 
$\bbeta \in \Omega$ of $\bbeta_*$, which may have certain nice properties that will become clear later on. 
Nevertheless, for the purpose of understanding the main idea, it may be convenient to 
simply assume that $\bbeta=\bbeta_*$ (thus $\bbeta_* \in \Omega$) during the first reading.

Given any $\bbeta \in \Omega$ and subset $G \subset \partial R(\bbeta)$, we define a modified regularizer
\[
R_G(\beta) = R(\bbeta) + \sup_{v \in G} \innerprod{v}{\beta-\bbeta} .
\]
It is clear that $R_G(\beta) \leq R(\beta)$ for all $\beta$ and $R(\bbeta)=R_G(\bbeta)$.
The value of $R_G(\beta)$ is unchanged if $G$ is replaced by the closure of its convex hull. 
Moreover, if $G$ is convex and closed, then the sub-differential of $R_G(\beta)$ is 
identical to $G$ at $\bbeta$ and contained in $G$ elsewhere. In fact, by checking 
the condition $R_G(b)-R_G(\beta)\ge \innerprod{v}{b-\beta}$ for $b=t\beta$ and $b=\bbeta$, 
we see that for closed convex $G$ 
\bes
\pa R_G(\beta) = \big\{v\in G:  R_G(\beta)=\innerprod{v}{\beta} = 
R(\bbeta)+\innerprod{v}{\beta-\bbeta} \big\}. 
\ees 
In what follows, we pick a closed convex $G$ unless otherwise stated. 

In optimization, $\beta$ is generally referred to as primal variable and $\nabla L(\beta)$ 
as the corresponding dual variable, 
since they live in $\Omegabar$ and $\Omegabar^*$ respectively. 
An optimal solution $\hbeta$ of (\ref{eq:hbeta}) satisfies the KKT condition when its dual
satisfies the relationship $-\nabla L(\hbeta) \in \partial R(\hbeta)$. However, for the general formulation
(\ref{eq:hbeta}), this condition can be rather hard
to work with.  Therefore in order to analyze (\ref{eq:hbeta}), we introduce the notion of 
{\em primal-dual certificate},
which is a primal variable $Q_G$ satisfying a simplified
dual constraint $-\nabla L(Q_G) \in \partial R(\bbeta)$. To be consistent with some 
earlier literature, one may 
refer to the quantity $-\nabla L(Q_G)$ as the corresponding dual certificate. 
For notational simplicity, without causing confusion, in this paper we will also refer to 
$Q_G$ as a dual certificate.

\subsection{Primal Dual Certificate Sparse Recovery Bound}

The formal definition of dual certificate is given in Definition~\ref{def:primal-dual-certificate}.
In this definition, we also allow approximate dual certificate which may have a small violation of
the dual constraint; such an approximation can be convenient
for some applications.
\begin{definition}[Primal-Dual Certificate]
  \label{def:primal-dual-certificate}
  Given any $\bbeta \in \Omega$ and a closed convex subset $G \subset \partial R(\bbeta)$.
  A $\delta$-approximate primal-dual (or simply dual) certificate $Q_G$ (with respect to $G$) of (\ref{eq:hbeta}) is a primal variable that satisfies the following condition:
  \begin{equation}
    - \nabla L(Q_G) + \delta \in G . \label{eq:dual-certificate}
  \end{equation}
If $\delta=0$, we call $Q_G$ an exact primal-dual certificate or simply a dual certificate.
\end{definition}
We may choose a convex function $\bar{L}(\beta)$ that
is close to $L(\beta)$ 
and use it to construct an approximate dual certificate with
\begin{equation}
  Q_G = \argmin_\beta\big\{\bar{L}(\beta)+R_G(\beta)\big\}. 
  \label{eq:dual-certificate-simple}
\end{equation}
Since $- \nabla \bar{L}(Q_G) \in \pa R_G(Q_G)\subseteq G$, 
(\ref{eq:dual-certificate}) holds for $\delta = \nabla\bar{L}(Q_G)-\nabla L(Q_G)$. 
However, this choice may not always lead to the best result in the analysis of 
the estimator (\ref{eq:hbeta}), especially when 
$- \nabla L(Q_G) + \delta = - \nabla \bar{L}(Q_G)$ is an interior 
point of $G$. Possible choices of $\bar{L}(\beta)$ include $\gamma L(\beta)$ with 
a constant $\gamma$, its expectation, and their approximations. 
Note that we do not assume that $Q_G \in \Omega$. In order to approximately enforce such a constraint, we may
replace $L(\beta)$ by $L(\beta) + L_\Delta(\beta)$ for any convex function $L_\Delta(\beta) \geq 0$ such that  
$L_\Delta(\beta)=0$ when $\beta \in \Omega$. If $L_\Delta(\beta)$ is sufficiently large, then we can construct a
$Q_G$ that is approximately contained in $\Omega$. 
More detailed dual certificate construction techniques are discussed in Section~\ref{sec:construction}.

An essential result that relates a primal-dual certificate $Q_G$ to $\hbeta$ is stated in the following fundamental
theorem, which says that if $Q_G$ is close to $\bbeta$, then $\hbeta$ is close to $\bbeta$ (when $\delta=0$).
In order to apply this theorem, we shall choose $\bbeta \approx \bbeta_*$.
\begin{theorem}[Primal-Dual Certificate Sparse Recovery Bound]
  Given an approximate primal-dual certificate $Q_G$ in Definition~\ref{def:primal-dual-certificate}, 
  we have the following inequality:
  \[
  D_L(\bbeta,\hbeta) + D_L(\hbeta,Q_G) + [ R(\hbeta)- R_G(\hbeta)] \leq D_L(\bbeta,Q_G) - \innerprod{\delta}{\hbeta-\bbeta} .
  \]
  \label{thm:dual_certificate-recovery}
\end{theorem}
The proof is a simple application of the following two propositions.

\begin{proposition}
\label{prop:bregman}
For any convex function $L(\cdot)$, the following identity holds for Bregman divergence:
\[
D_L(a,b) + D_L(b,c) - D_L(a,c)= \innerprod{\nabla L(c) - \nabla L(b)}{a-b} .
\]
\end{proposition}
\begin{proof} 
This can be easily verified using simple algebra. We can expand the left hand side as follows.
\begin{align*}
 & D_L(a,b) + D_L(b,c) - D_L(a,c) \\
=& \left[ L(a) - L(b) - \innerprod{\nabla L(b)}{a-b} \right] 
 + \left[ L(b) - L(c) - \innerprod{\nabla L(c)}{b-c} \right] 
- \left[ L(a) - L(c) - \innerprod{\nabla L(c)}{a-c} \right] \\
=& - \innerprod{\nabla L(b)}{a-b} 
- \innerprod{\nabla L(c)}{b-c} + \innerprod{\nabla L(c)}{a-c} . 
\end{align*}
This can be simplified to obtain the right hand side.
\end{proof}
\begin{proposition}
\label{prop:subgrad}
Let $\tbeta={t} \hbeta + (1-{t}) \bbeta$ for some ${t} \in [0,1]$.
Then, given any $v \in G$, we have
\[
\innerprod{- v - \nabla L(\tbeta)}{\bbeta - \tbeta} \leq R_G(\tbeta) - R(\tbeta) .
\]
\end{proposition}
\begin{proof}
The definition of $\hbeta$ and the convexity of (\ref{eq:hbeta})
imply that $\tbeta$ achieves the minimum objective
value $L(\beta)+R(\beta)$ for $\beta$ that lies in the line segment between
$\tbeta$ and $\bbeta$. This is equivalent to
$\innerprod{\nabla L(\tbeta) + \nabla R(\tbeta)}{\bbeta-\tbeta}\geq 0$.
Since $R(\cdot)$ is convex, this implies
$\innerprod{\nabla L(\tbeta)}{\bbeta-\tbeta} + R(\bbeta) \ge R(\tbeta)$. 
Thus, 
\bes
\innerprod{-v-\nabla L(\tbeta)}{\bbeta-\tbeta}\le \innerprod{v}{\tbeta-\bbeta}+R(\bbeta)-R(\tbeta)
\le R_G(\tbeta)-R(\tbeta)
\ees
by the definition of $R_G(\beta)$. 
\end{proof}

\begin{proof}{\bf of Theorem~\ref{thm:dual_certificate-recovery}}.
We apply Proposition~\ref{prop:bregman} with $a=\bbeta$, $b=\hbeta$, and $c=Q_G$ to obtain:
\[
 D_L(\bbeta,\hbeta) + D_L(\hbeta,Q_G) - D_L(\bbeta,Q_G) = \innerprod{\nabla L(Q_G)  - \nabla L(\hbeta)}{\bbeta-\hbeta} 
=\innerprod{-v+\delta  - \nabla L(\hbeta)}{\bbeta-\hbeta} ,
\]
where $v \in G$. We can now apply Proposition~\ref{prop:subgrad} with $t=1$ to obtain the desired bound.
\end{proof}

The results shows that if we have a good bound on $D_L(\bbeta,Q_G)$, then it is possible to 
obtain a bound on $D_L(\bbeta,\hbeta)$. In general, we
also choose $G$ so that the difference $R(\hbeta)-R_G(\hbeta)$ can effectively
control the magnitude of $\hbeta$ outside of the support (or a tangent space) of $\bbeta$.

\subsection{Primal Dual Certificate Sparse Oracle Inequality}

It is also possible to derive a stronger form of oracle inequality for special $L$ with a more refined definition
of dual certificate.
\begin{definition}[Generalized Primal-Dual Certificate]
  \label{def:primal-dual-certificate-2}
  Given $\bbeta \in \Omega$,
a closed convex set $G \subset \partial R(\bbeta)$, a convex function $\bar{L}$ on $\Omegabar$, and an additional parameter $\obeta \in \Omegabar$.
  A generalized $\delta$-approximate primal-dual (or simply dual) certificate $Q_G$ with respect to 
  $(L,\bar{L},\bbeta,\obeta)$ is a primal variable that satisfies the following condition:
  \begin{equation}
    - \nabla \bar{L}_*(Q_G) +\delta \in G , \label{eq:dual-certificate-2}
  \end{equation}
  where $\bar{L}_*(\beta)= \bar{L}(\beta) - \innerprod{\nabla \bar{L}(\bbeta)-\nabla L(\obeta)}{\beta-\bbeta}$.
\end{definition}

Note that if $\innerprod{\cdot}{\cdot}$ is an inner product and 
$L$ is a quadratic function of the form
\begin{equation}
L(\beta) = \innerprod{H\beta - z}{\beta}
\label{eq:quadratic-loss}
\end{equation}
for some self-adjoint operator $H$ and vector $z$, then $D_L(\beta,\beta')=\innerprod{H (\beta - \beta')}{\beta-\beta'}$.
In this case, we may simply take $\bar{L}(\cdot)= L(\cdot)$. 
For other cost functions, it will be useful to take $\bar{L}(\cdot)=\gamma L(\cdot)$ with $\gamma <1$.
The reason will become clear later on.

Definition~\ref{def:primal-dual-certificate-2} is equivalent to 
Definition~\ref{def:primal-dual-certificate} with $L(\beta)$ replaced by
a redefined convex function $\bar{L}_*(\beta)= \bar{L}(\beta) - \innerprod{\nabla \bar{L}(\bbeta)-\nabla L(\obeta)}{\beta-\bbeta}$.
We may consider $\obeta$ to be the true target $\bbeta_*$ (or its approximation)
in that we can assume that $\nabla L(\obeta)$ is small although 
$\obeta$ may not be sparse.
The main advantage of Definition~\ref{def:primal-dual-certificate-2} is that it allows comparison to an
arbitrary sparse approximation $\bbeta$ to $\obeta$ even when $\nabla L(\bbeta)$ is 
not small --- the definition only requires  $\nabla \bar{L}_*(\bbeta)=\nabla L(\obeta)$ to be small. 
This implies that $\bbeta$ may have a dual certificate $Q_G$ with respect to $\bar{L}_*(\cdot)$
that is close to $\bbeta$  (see error bounds in Section~\ref{sec:construction}). 
The following result shows that one can obtain an oracle inequality that generalizes Theorem~\ref{thm:dual_certificate-recovery}. In order to apply this theorem, we should choose $\obeta \approx \bbeta_*$.
\begin{theorem}[Primal-Dual Certificate Sparse Oracle Inequality]
  Given a generalized $\delta$ approximate primal-dual certificate $Q_G$ in Definition~\ref{def:primal-dual-certificate-2}, 
  we have for all $\tbeta$ in the line segment between $\bbeta$ and $\hbeta$:
\bes
  && D_L(\bbeta,\tbeta)+  D_L(\tbeta,\obeta) + D_{\bar{L}}(\tbeta,Q_G) + [R(\tbeta)- R_G(\tbeta) ]
  \cr &\leq& D_{\bar{L}}(\tbeta,\bbeta)+ D_L(\bbeta,\obeta) + D_{\bar{L}}(\bbeta,Q_G)  
  - \innerprod{\delta}{\tbeta-\bbeta}.
\ees
\label{thm:dual_certificate-oracle}
\end{theorem}
\begin{proof}
We apply Proposition~\ref{prop:bregman} with $a=\bbeta$, $b=\tbeta$, and $c=\obeta$ to obtain:
\[
D_L(\bbeta,\tbeta)+ D_L(\tbeta,\obeta) - D_L(\bbeta,\obeta) =\innerprod{\nabla L(\obeta) - \nabla L(\tbeta)}{\bbeta-\tbeta}.
\]
Similarly, we can apply Proposition~\ref{prop:bregman} with $a=\tbeta$, $b=\bbeta$, and $c=Q_G$ 
to $\bar{L}$ to obtain:
\[
D_{\bar{L}}(\tbeta,\bbeta)+ D_{\bar{L}}(\bbeta,Q_G) - D_{\bar{L}}(\tbeta,Q_G)
=\innerprod{\nabla \bar{L}(Q_G) - \nabla \bar{L}(\bbeta)}{\tbeta-\bbeta}.
\]
By subtracting the above two displayed equations, we obtain
\begin{align*}
&D_L(\bbeta,\tbeta)+ D_L(\tbeta,\obeta) - D_L(\bbeta,\obeta) 
-D_{\bar{L}}(\tbeta,\bbeta)- D_{\bar{L}}(\bbeta,Q_G) + D_{\bar{L}}(\tbeta,Q_G)  \\
=&
\innerprod{\nabla L(\obeta) - \nabla L(\tbeta)
+\nabla \bar{L}(Q_G) - \nabla \bar{L}(\bbeta)}{\bbeta-\tbeta}. 
\end{align*}
Since $\nabla \bar{L}(Q_G) + \nabla L(\obeta) - \nabla \bar{L}(\bbeta)=-v +\delta$ 
for some $v\in G$, the right hand side can be written as
$\innerprod{-v+\delta-\nabla L(\tbeta)}{\bbeta-\tbeta}$. 
The conclusion then follows from Proposition~\ref{prop:subgrad}. \end{proof}

Note that if we choose $L=\bar{L}$ and $\obeta=\bbeta$ in Theorem~\ref{thm:dual_certificate-oracle},
then Definition~\ref{def:primal-dual-certificate-2} is consistent with
Definition~\ref{def:primal-dual-certificate}, and Theorem~\ref{thm:dual_certificate-oracle} becomes
Theorem~\ref{thm:dual_certificate-recovery}. 
Since $\bar{L}_*(\beta)-\bar{L}(\beta)$ is linear in $\beta$, $D_{\bar{L}}(\bbeta,Q_G)
= D_{\bar{L}_*}(\bbeta,Q_G)$. Moreover, when $\nabla L(\beta_*)$ is small, 
$\nabla\bar{L}_*(\bbeta)$ is small by the choice of $\bar{L}_*(\cdot)$ in 
Definition~\ref{def:primal-dual-certificate-2}, so that $D_{\bar{L}_*}(\bbeta,Q_G)$ is 
small when $\bar{L}_*$ has sufficient convexity near $\bbeta$. 
This motivates a choice $\bar{L}(\cdot)$ satisfying 
$D_{\bar{L}}(\beta,\bbeta)\le D_L(\bbeta,\beta)$ for all $\beta\in\Omega$ whenever 
such a choice is available and reasonably convex near $\bbeta$. 
This lead to the following corollary. 

\begin{corollary} \label{cor:dual_certificate-oracle}
Given a generalized exact primal-dual certificate $Q_G$ in Definition~\ref{def:primal-dual-certificate-2}
with $\bar{L}(\cdot)$ 
  satisfying $D_L(\bbeta,\beta)\ge D_{\bar{L}}(\beta,\bbeta)$ for all $\beta\in\Omega$. Then, 
  \[
  D_L(\hbeta,\obeta) + [R(\hbeta)- R_G(\hbeta) ]
  \leq D_L(\bbeta,\obeta) + D_{\bar{L}_*}(\bbeta,Q_G). 
 \]
\end{corollary}

In some problems, Corollary \ref{cor:dual_certificate-oracle} is applicable with 
$\bar{L}(\cdot)=\gamma L(\cdot)$ for some $\gamma \in (0,1]$. 
In the special case that $L(\cdot)$ is a quadratic function as in (\ref{eq:quadratic-loss}), 
we have $D_L(\beta,\bbeta)=D_L(\bbeta,\beta)$. Therefore we may take $\gamma=1$,
and the bound in Corollary~\ref{cor:dual_certificate-oracle} can be further simplified to
\[
D_L(\hbeta,\obeta)+ [R(\hbeta)- R_G(\hbeta) ] \leq 
D_L(\bbeta,\obeta)+  D_{L}(\bbeta,Q_G)  .
\]

If $L(\cdot)$ comes from a generalized linear model of the form 
$L(\beta)=\sum_{i=1}^n \ell_i(\innerprod{x_i}{\beta})$, with $x_i\in\Omegabar^*$ and second order differentiable convex scalar functions $\ell_i$, then the condition $D_L(\bbeta,\beta) \geq \gamma D_L(\beta,\bbeta)$ is satisfied as long as:
\[
\inf_{\beta \in \Omega} \frac{D_L(\bbeta,\beta)}{D_L(\beta,\bbeta)}
\geq \inf_{\{\beta',\beta''\} \in \Omega} 
\frac{\sum_{i=1}^n \ell_i''(\innerprod{x_i}{\beta'}) \innerprod{x_i}{\beta-\bbeta}^2}
{\sum_{i=1}^n \ell_i''(\innerprod{x_i}{\beta''}) \innerprod{x_i}{\beta-\bbeta}^2}
\geq \inf_{i \in \{1,\ldots,n\}} \inf_{\{\beta',\beta''\} \in \Omega}\frac{\ell_i''(\innerprod{x_i}{\beta'})}{\ell_i''(\innerprod{x_i}{\beta''})}\ge\gamma. 
\]
This means that the condition of Corollary~\ref{cor:dual_certificate-oracle} holds as long as for all $i,\beta,\beta' \in \Omega$:
$\ell_i''(\innerprod{x_i}{\beta}) \geq \gamma \ell_i''(\innerprod{x_i}{\beta'})$.
For example, for logistic regression $\ell_i(t)= \ln(1+\exp(-t))$ 
with $\sup_i\sup_{\beta\in\Omega}|\innerprod{x_i}{\beta}|\leq A$, we can pick 
$\gamma=4/(2+\exp(-A)+\exp(A))$. This choice of $\gamma$ can be improved if we have additional constraints
on $\hat{\beta}$; an example is given in Corollary~\ref{cor:recovery-global-dc-oracle}.
In~\ref{sec:genlin-example}, we will present a more concrete and elaborated analysis for generalized linear models.

Note that the result of Corollary~\ref{cor:dual_certificate-oracle} gives an oracle inequality that compares
$D_L(\hbeta,\obeta)$ to $D_L(\bbeta,\obeta)$ with leading coefficient one. 
The bound is meaningful as long as $\bbeta$ has a good dual certificate $Q_G$ 
under $\bar{L}_*(\beta)$ that is close to $\bbeta$.
The possibility to obtain oracle inequalities of this kind with leading coefficient one
was first noticed in \cite{KoTsLo10} under restricted strong convexity. 
The advantage of such an oracle inequality is that we do not require $\obeta$ to be sparse,
but rather the competitor $\bbeta$ to be sparse --- which implies the dual certificate $Q_G$ 
is close to $\bbeta$ when $\bar{L}_*(\beta)$ is sufficiently convex.
Here we generalize the result of \cite{KoTsLo10} in two ways. 
First it is possible to deal with non-quadratic loss.
Second we only require the existence of a good dual certificate $Q_G$,
which is a weaker requirement than restricted strong convexity in \cite{KoTsLo10}.

Generally speaking, the dual certificate technique allows us to obtain oracle inequality
$D_L(\hbeta,\obeta)+ [R(\hbeta)- R_G(\hbeta) ]$ directly.
If we are interested in other results such as parameter estimation bound 
$\|\hbeta-\obeta\|$, then additional estimates will be needed on top of the dual certificate theory of this paper.
Instead of working out general results, we will study this problem for structured $\ell_1$ regularizer
in Section~\ref{sec:struct-L1}.

\section{Constructing Primal-Dual Certificate}
\label{sec:construction}
We will present some general results for estimating $D_L(\bbeta,Q_G)$ under various assumptions.
For notational
simplicity, the main technical derivation considers Definition~\ref{def:primal-dual-certificate}, with dual certificate
$Q_G$ with respect to $L(\beta)$. One can then apply these results to the dual certificate $Q_G$
in Definition~\ref{def:primal-dual-certificate-2}. 

\subsection{Global Restricted Strong Convexity} 
\label{sec:RSC}
We first consider the following construction of primal-dual certificate.
\begin{proposition} \label{prop:global-dc}
Let
\begin{equation}
Q_G = \arg\min_{\beta} \left[ L(\beta) + R_G(\beta) \right] , \label{eq:dc-opt}
\end{equation}
then $Q_G$ is an exact primal-dual certificate of (\ref{eq:hbeta}).
\end{proposition}
\begin{proof}
It is clear from the optimality condition of (\ref{eq:dc-opt}) that $\nabla L(Q_G) + v=0$ 
for some $v \in G$. 
\end{proof}

The symmetrized Bregman divergence is  defined as 
\bes
D_L^s(\beta,\bbeta)=D_L(\beta,\bbeta) + D_L(\bbeta,\beta) 
= \innerprod{\nabla L(\beta) - \nabla L(\bbeta)}{\beta-\bbeta}. 
\ees
We introduce the concept of restricted strong convexity to bound 
$D_L^s(\bbeta,Q_G)$.
\begin{definition}[Restricted Strong Convexity]\label{def:RSC}
We define the following quantity which we refer to as global restricted strong convexity (RSC) constant:
\[
 \gamma_L(\bbeta;r,G,\|\cdot\|)= \inf \left\{ \frac{D_L^s(\beta,\bbeta)}{\|\beta-\bbeta\|^2} : 
0<\|\beta-\bbeta\|\leq r; \;
D_L^s(\beta,\bbeta)+\sup_{u\in G} \innerprod{u+\nabla L(\bbeta)}{\beta-\bbeta} \leq 0 \right\} ,
\]
where $\|\cdot\|$ is a norm in $\Omegabar$, $r>0$ and $G \subset \partial R(\bbeta)$.
\end{definition}
The parameter $r$ is introduced for localized analysis, where the Hessian may be small when 
$\|\beta-\bbeta\| >r$. For least squares loss that has a constant Hessian, one can just pick $r=\infty$.

We recall the concept of dual norm in $\Omegabar$: 
$\|\cdot\|_D$ is the dual norm of $\|\cdot\|$ if
\[
\|u\|_D = \sup_{\|\beta\|=1} \innerprod{u}{\beta} . 
\]
It implies the inequality that $\innerprod{u}{\beta} \leq \|u\|_D \|\beta\|$.

\begin{theorem}[Dual Certificate Error Bound under RSC]
Let $\|\cdot\|$ be a norm in $\Omegabar$ and $\|\cdot\|_D$ its dual norm in $\Omegabar^*$. 
Consider $\bbeta \in \Omega$ and a closed convex $G \subset \partial R(\bbeta)$.
Let $\Delta_r=\gamma_L(\bbeta;r,G,\|\cdot\|)^{-1}\inf_{u  \in G} \|u+\nabla L(\bbeta)\|_D$.
If $\Delta_r < r$ for some $r>0$, then for any $Q_G$ given by  (\ref{eq:dc-opt}), 
\[
D_L^s(\bbeta,Q_G) \leq \gamma_L(\bbeta;r,G,\|\cdot\|) \Delta_r^2, \quad \|\bbeta-Q_G\| \leq \Delta_r . 
\]
\label{thm:dual_certificate-error}
\end{theorem}

\begin{proof}
By the optimality condition (\ref{eq:dc-opt}) of $Q_G$, there exists $v\in\pa R_G(Q_G)$ such that 
$\nabla L(Q_G) + v = 0$. 
For $v\in\pa R_G(Q_G)$, $R_G(Q_G)-R(\bbeta)=\innerprod{v}{Q_G-\bbeta} 
\geq \sup_{u\in G}\innerprod{u}{Q_G-\bbeta}$. Therefore,
\bel{eq:Q-RSC}
D_L^s(Q_G,\bbeta) 
= \innerprod{\nabla L(Q_G)-\nabla L(\bbeta)}{Q_G-\bbeta} \leq 
-\innerprod{u+\nabla L(\bbeta)}{Q_G-\bbeta}, \forall \ u \in G. 
\eel
Let ${\tilde Q}_G=\bbeta+t(Q_G-\bbeta)$ where we pick $t=1$ if $\|Q_G-\bbeta\| \leq r$ and $t \in (0,1)$
with $\|{\tilde Q}_G-\bbeta\|=r$ otherwise. 
Let $f(t) = D_L({\tilde Q}_G,\bbeta)$ so that $D_L^s({\tilde Q}_G,\bbeta) = tf'(t)$. 
The convexity of $L(\beta)$ implies $f'(t) \le f'(1)=D_L^s(Q_G,\bbeta)$. 
It follows that  
\bes
D_L^s({\tilde Q}_G,\bbeta) +\innerprod{u+\nabla L(\bbeta)}{{\tilde Q}_G-\bbeta}
\le t\{D_L^s(Q_G,\bbeta)+\innerprod{u+\nabla L(\bbeta)}{Q_G-\bbeta}\}\le 0, 
\ees
which implies the restricted cone condition for ${\tilde Q}_G$ in the definition of RSC. Thus, 
\[
\gamma_L(\bbeta;r,G,\|\cdot\|) \|{\tilde Q}_G-\bbeta\|^2 -\|u+\nabla L(\bbeta)\|_D \|{\tilde Q}_G-\bbeta\|
\leq 0 .
\]
Now by moving the term $\|u+\nabla L(\bbeta)\|_D \|{\tilde Q}_G-\bbeta\|$ to the right hand side and
taking $\inf$ over $u$, we obtain
$\gamma_L(\bbeta;r,G,\|\cdot\|)\|{\tilde Q}_G-\bbeta\| \leq \inf_{u\in G} \|u+\nabla L(\bbeta)\|_D
=\gamma_L(\bbeta;r,G,\|\cdot\|) \Delta_r$. 
Since $\Delta_r < r$, we have $t=1$ and ${\tilde Q}_G=Q_G$.
It means that we always have $\|Q_G-\bbeta\| \leq \Delta_r <r$. Consequently, 
(\ref{eq:Q-RSC}) gives 
$D_L^s(Q_G,\bbeta) \leq \inf_{u\in G}\|u+\nabla L(\bbeta)\|_D\Delta_r$. 
This completes the proof.
\end{proof}

\begin{remark}
  Although for simplicity, the proof of Theorem~\ref{thm:dual_certificate-error} implicitly assumes that the solution of (\ref{eq:dc-opt}) is finite,
  this extra assumption is not necessary with a slightly more complex argument (which we excludes in the proof in order not to obscure the main idea). An easy way to see this is by
  adding a small (unrestricted) strongly convex term $L_\Delta(\beta)$ to $L$ and consider dual certificate for the modified function
  $\tilde{L}(\beta)=L(\beta)+L_\Delta(\beta)$.
  Since the solution of (\ref{eq:dc-opt}) with $\tilde{L}(\beta)$ is finite, we can apply the proof to $\tilde{L}(\beta)$ and then simply let $L_\Delta(\beta) \to 0$. 
\end{remark}

Note that if $\nabla L(Q_G)$ is not unique, then the same value can be used both in Theorem~\ref{thm:dual_certificate-recovery} and in
Theorem~\ref{thm:dual_certificate-error}. 
Since $D_L(\bbeta,Q_G)\le D_L^s(\bbeta,Q_G)$, this implies the following bound:

\begin{corollary}
  Under the conditions of Theorem~\ref{thm:dual_certificate-error}, we have
  \[
  D_L(\bbeta,\hbeta) + [ R(\hbeta)- R_G(\hbeta)] \leq \gamma_L(\bbeta;r,G,\|\cdot\|)^{-1}\inf_{u  \in G} \|u+\nabla L(\bbeta)\|_D^2 .
  \]
 \label{cor:recovery-global-dc-error}
\end{corollary}
Similarly, we may apply Theorem~\ref{thm:dual_certificate-oracle} and Theorem~\ref{thm:dual_certificate-error}
with $L(\beta)$ replaced by $\bar{L}_*(\beta)$ as in Definition~\ref{def:primal-dual-certificate-2}.
This implies the following general recovery bound.

\begin{corollary}
  Let $\|\cdot\|$ be a norm in $\Omegabar$ and $\|\cdot\|_D$ its dual norm in $\Omegabar^*$. 
  Consider $\bbeta \in \Omega$ and a closed convex $G \subset \partial R(\bbeta)$.
  Consider $\bar{L}(\beta)$ as in Definition~\ref{def:primal-dual-certificate-2}, and define 
  \[
  \gamma_{\bar{L}_*}(\bbeta;r,G,\|\cdot\|)= \inf \left\{ \frac{D_{\bar{L}}^s(\beta,\bbeta)}{\|\bbeta-\beta\|^2} : 
    \|\beta-\bbeta\|\leq r; \;
    D_{\bar{L}}^s(\beta,\bbeta)+\sup_{u \in G} \innerprod{u+\nabla L(\obeta)}{\beta-\bbeta} \leq 0 \right\}
  \]
  and 
$\Delta_r =   (\gamma_{\bar{L}_*}(\bbeta;r,G,\|\cdot\|))^{-1}\inf_{u  \in G} \|u+\nabla L(\obeta)\|_D$. 
  Assume for some $r>0$, we have   $\Delta_r < r$; and assume
  there exists $\tilde{r}>  D_L(\bbeta,\obeta)+  \gamma_{\bar{L}_*}(\bbeta;r,G,\|\cdot\|) \Delta_r^2$ such that for all $\beta \in \Omega$:
  $D_L(\beta,\obeta) + [R(\beta)- R_G(\beta) ] \leq \tilde{r}$ implies
  $D_L(\bbeta,\beta) \geq D_{\bar{L}}(\beta,\bbeta)$.
  Then, 
  \[
  D_L(\hbeta,\obeta) + [R(\hbeta)- R_G(\hbeta) ] \leq 
  D_L(\bbeta,\obeta)+  \gamma_{\bar{L}_*}(\bbeta;r,G,\|\cdot\|) \Delta_r^2.
\]
  \label{cor:recovery-global-dc-oracle}
\end{corollary}

\begin{proof} 
Let $\bar{L}_*(\beta)= \bar{L}(\beta) - \innerprod{\nabla \bar{L}(\bbeta)-\nabla L(\obeta)}{\beta-\bbeta}$ 
and define
\[
Q_G = \arg\min_{\beta} \left[ \bar{L}_*(\beta) + R_G(\beta) \right]  .
\]
Then $Q_G$ is a generalized dual certificate in Definition~\ref{def:primal-dual-certificate-2}.
Note that $D_{\bar{L}_*}(\beta,\beta')=D_{\bar{L}}(\beta,\beta')$ and $\nabla \bar{L}_*(\bbeta)=\nabla L(\obeta)$. 
The conditions of the corollary and Theorem~\ref{thm:dual_certificate-error}, 
applied with $L$ replaced by $\bar{L}_*$, imply that
$\|Q_G-\bbeta\| \leq r$ and $D_{\bar{L}}(\bbeta,Q_G) \leq \gamma_{\bar{L}_*}(\bbeta;r,G,\|\cdot\|) \Delta_r^2$. Now we simply apply Theorem~\ref{thm:dual_certificate-oracle} to obtain 
that for all $t \in [0,1]$ and 
$\tbeta= \bbeta + t (\hbeta-\bbeta)$:
\[
  D_L(\tbeta,\obeta) + [R(\tbeta)- R_G(\tbeta) ] \leq 
  D_L(\bbeta,\obeta)+  \gamma_{\bar{L}_*}(\bbeta;r,G,\|\cdot\|) \Delta_r^2  +  [D_{\bar{L}}(\tbeta,\bbeta)-D_L(\bbeta,\tbeta)] .
\]
It is clear that when $t=0$, we have $D_L(\tbeta,\obeta) + [R(\tbeta)- R_G(\tbeta) ] < \tilde{r}$.
If the condition $D_L(\tbeta,\obeta) + [R(\tbeta)- R_G(\tbeta) ] \leq \tilde{r}$ holds for $t=1$, 
then the desired bound is already proved due to the condition 
$D_L(\bbeta,\tbeta) \geq D_{\bar{L}}(\tbeta,\bbeta)$. 
Otherwise, there exists $t \in [0,1]$ such that 
$D_L(\tbeta,\obeta) + [R(\tbeta)- R_G(\tbeta) ] = \tilde{r}$. However, this is impossible because 
the same argument gives 
\[
  D_L(\tbeta,\obeta) + [R(\tbeta)- R_G(\tbeta) ] \leq 
  D_L(\bbeta,\obeta)+  \gamma_{\bar{L}_*}(\bbeta;r,G,\|\cdot\|) \Delta_r^2 < \tilde{r} .
\]
This proves the desired bound.
\end{proof}

Corollary~\ref{cor:recovery-global-dc-oracle} gives an oracle inequality with leading coefficient one
for general loss functions, but the statement is rather complex.
The situation for quadratic loss is much simpler, where we can take $\bar{L}(\beta)=L(\beta)$. 
This is because the condition $D_L(\bbeta,\tbeta) \geq D_{\bar{L}}(\tbeta,\bbeta)$ always holds.
We also have a better constant because $D_L^s(\beta,\beta')= 2 D_L(\beta,\beta')=2D_L(\beta',\beta)$.
\begin{corollary}
  Assume that $L(\beta)$ is a quadratic loss in (\ref{eq:quadratic-loss}).
  Let $\|\cdot\|_D$ and $\|\cdot\|$ be dual norms, and consider $\bbeta \in \Omega$ and a closed 
  convex $G \subset \partial R(\bbeta)$. We have
  \[
  D_L(\hbeta,\obeta)+ [R(\hbeta)- R_G(\hbeta) ] \leq 
  D_L(\bbeta,\obeta)+  
  (2\gamma_{\bar{L}_*}(\bbeta;\infty,G,\|\cdot\|))^{-1}\inf_{u  \in G} \|u+\nabla L(\obeta)\|_D^2 ,
  \]
 where
 \[
  \gamma_{\bar{L}_*}(\bbeta;\infty,G,\|\cdot\|)= \inf \left\{ \frac{2D_L(\beta,\bbeta)}{\|\bbeta-\beta\|^2} : 
    2 D_L(\beta,\bbeta)+\sup_{u \in G} \innerprod{u+\nabla L(\obeta)}{\beta-\bbeta} \leq 0 \right\} .
 \]
\label{cor:recovery-global-dc-quadratic}
\end{corollary}

\subsection{Quadratic Loss with Gaussian Random Design Matrix}
\label{sec:gordon}
While in the general case, the estimation of $\gamma_{\bar{L}_*}(\bbeta;r,G,\|\cdot\|)$ may be technically involved, 
for the special application of compressed sensing with Gaussian random design matrix and quadratic loss, 
we can obtain a relatively general and simple bound using 
Gordon's minimum restricted singular value estimation in \cite{Gordon88}. This section describes the underlying idea.

In this section, we consider the quadratic loss function
\begin{equation}
L(\beta) = \|X \beta - Y\|_2^2 ,
\label{eq:gaussian-design}
\end{equation}
where $\beta \in \Real^p, Y \in \Real^n$, and $X$ is an $n \times p$ matrix with iid Gaussian entries $N(0,1)$. 
Here $\innerprod{\cdot}{\cdot}$ is the Euclidean dot product in $\Real^p$:
$\innerprod{u}{v}= u^\top v$ for $u,v \in \Real^p$.
\begin{definition}[Gaussian Width]
  Given any set $\cC \subset \Real^p$, we define its Gaussian width as
  \[
  \gw(\cC) = \rE_{\epsilon} \sup_{z \in \cC; \|z\|_2=1} \epsilon^\top z ,
  \]
  where $\epsilon \sim N(0, I_{p \times p})$ and $\rE_{\epsilon}$ is the expectation with respect to $\epsilon$.
\end{definition}

The following estimation of Gaussian width is based on a similar computational technique used in \cite{ChRePaWi10}.
\begin{proposition}\label{prop:gw}
  Let 
 $\cC=\left\{\beta \in \Real^p : \sup_{u \in G} \innerprod{u+\nabla L(\obeta)}{\beta} \leq 0 \right\}$ 
 and $\epsilon \sim N(0,I_{p \times p})$. Then,  
 \[
 \gw(\cC) \leq \rE_{\epsilon} \inf_{u \in G; \gamma >0} \|\gamma(u+\nabla L(\obeta)) - \epsilon\|_2. 
 \]
\end{proposition}
\begin{proof}
  For all $\beta \in \cC$ and $\|\beta\|_2=1$, $\gamma \geq 0$, and $u \in G$,
  let $g= (u+\nabla L(\obeta))$. We have
  $\innerprod{g}{\beta}=\innerprod{u + \nabla L(\obeta)}{\beta} \leq 0$.
  Therefore, 
$\epsilon^\top \beta = 
(\epsilon- \gamma g)^\top \beta
+\gamma g^\top \beta
\leq (\epsilon- \gamma g)^\top \beta \leq \|\epsilon- \gamma g\|_2$. 
Since $u$ is arbitrary, we have
\[
\epsilon^\top \beta \leq \inf_{u \in G; \gamma >0} \|\gamma(u+\nabla L(\obeta)) - \epsilon\|_2 .
\]
Taking expectation with respect to $\epsilon$, we obtain the desired result.
\end{proof}

Gaussian width is useful when we apply Gordon's restricted singular value estimates, 
which give the following result.
\begin{theorem}
\label{thm:gordon}
Let $f_{\min}(X)=\min_{z \in \cC; \|z\|_2=1} \|X z\|_2$ and $f_{\max}(X)=\max_{z \in \cC; \|z\|_2=1} \|X z\|_2$. Let $\lambda_n=\sqrt{2} \Gamma((n+1)/2)/\Gamma(n/2)$ where $\Gamma(\cdot)$ is the $\Gamma$-function. We have for any $\delta >0$:
\[
\rP \left[ f_{\min}(X) \leq \lambda_n - \gw(\cC) - \delta \right]
\leq \rP[ N(0,1)>\delta ] \le 0.5 \exp \left(-\delta^2/2 \right) ,
\]
\[
\rP \left[ f_{\max}(X) \geq \lambda_n + \gw(\cC) + \delta \right]
\leq \rP[ N(0,1)>\delta ] \le 0.5 \exp \left(-\delta^2/2 \right). 
\]
\end{theorem}
\begin{proof}
  Since both $f_{\min}(X)$ and $f_{\max}(X)$ are Lipschitz-1 functions with respect to the Frobenius norm of $X$. 
  We may apply the Gaussian concentration bound \cite{Borell75,Pisier85} to obtain:
\[
\rP \left[ f_{\min}(X) \leq \rE [f_{\min}(X)] - \delta \right]
\leq \rP[ N(0,1)>\delta ],
\]
\[
\rP \left[ f_{\max}(X) \geq \rE [f_{\max}(X)] + \delta \right]
\leq \rP[ N(0,1)>\delta ].
\]
Now we may apply Corollary 1.2 of \cite{Gordon88} to obtain the estimates
\[
\rE [f_{\min}(X)] \geq \lambda_n - \gw(\cC) , \qquad
\rE [f_{\max}(X)] \leq  \lambda_n + \gw(\cC) ,
\]
which proves the theorem.
\end{proof}

Note that we have $n/\sqrt{n+1} \leq \lambda_n \leq \sqrt{n}$.
Therefore we may replace $\lambda_n-\gw(\cC)$ by $n/\sqrt{n+1}-\gw(\cC)$ and
$\lambda_n + \gw(\cC)$ by $\sqrt{n}+\gw(\cC)$.
By combining Theorem~\ref{thm:gordon} and Proposition~\ref{prop:gw} to estimate $\gamma_{\bar{L}_*}(\cdot)$ in 
Corollary~\ref{cor:recovery-global-dc-quadratic}, we obtain the following result for Gaussian random projection in compressed sensing. The result improves the main ideas of \cite{ChRePaWi10}.
\begin{theorem}\label{thm:recovery-gaussian}
  Let $L(\beta)$ be given by (\ref{eq:gaussian-design}) and $\epsilon\sim N(0,I_{p \times p})$. 
  Suppose the conditions of Theorem~\ref{thm:dual_certificate-error} hold. 
  Then, given any $g,\delta \geq 0$ such that $g+\delta \leq n/\sqrt{n+1}$, with probability at least
  \[
  1 - \frac{1}{2}\exp \left(-\frac{1}{2} (n/\sqrt{n+1}-g-\delta)^2\right) ,
  \]
  we have either 
  \[
  \|X(\hbeta-\obeta)\|_2^2 + [R(\hbeta)- R_G(\hbeta) ] \leq 
  \|X(\bbeta-\obeta)\|_2^2 + 
  (4\delta)^{-1}\inf_{u  \in G} \|u+\nabla L(\obeta)\|_2^2 ,
  \]
  or 
  \[
  g < \rE_{\epsilon} \inf_{u \in G; \gamma >0} \|\gamma(u+\nabla L(\obeta)) - \epsilon\|_2 .
  \]
\end{theorem}
\begin{proof}
  Let $\|\cdot\|=\|\cdot\|_D=\|\cdot\|_2$ in Corollary~\ref{cor:recovery-global-dc-quadratic}. 
  We simply note that $\gamma_{\bar{L}_*}(\bbeta;\infty,G,\|\cdot\|_2)$ is no smaller than 
  $\inf \{2\|X \beta\|_2 : \|\beta\|_2=1, \beta \in \cC\}$, where
  $\cC=\left\{\beta \in \Real^p : \sup_{u \in G} \innerprod{u+\nabla L(\obeta)}{\beta} \leq 0 \right\}$.
  Let $E_1$ be the event $g \geq \rE_{\epsilon} \inf_{u \in G; \gamma >0} \|\gamma(u+\nabla L(\obeta)) - \epsilon\|_2$. In the event $E_1$, Proposition~\ref{prop:gw} implies $g \geq \gw(\cC)$, 
  so that by Theorem~\ref{thm:gordon}
  \bes
  \rP\Big[\gamma_{\bar{L}_*}(\bbeta;\infty,G,\|\cdot\|_2)\le 2\delta \hbox{ and } E_1 \Big]
  \le \rP\Big[\inf \{\|X \beta\|_2 : \|\beta\|_2=1, \beta \in \cC\}\le \delta \Big| E_1 \Big]
  \le \frac{1}{2}e^{-(\lam_n-g-\delta)^2/2}. 
    \ees
The desired result thus follows from Corollary~\ref{cor:recovery-global-dc-quadratic}. 
\end{proof}

\begin{remark}
  If $Y= X \obeta + \epsilon$, with iid Gaussian noise $\epsilon \sim N(0,\sigma^2 I_{n \times n})$, then 
  the error bound in Theorem~\ref{thm:recovery-gaussian} depends on 
  $\inf_{u  \in G} \|u+\nabla L(\obeta)\|_2^2 = \inf_{u  \in G} \|u+2 X^\top X \epsilon\|_2^2 \approx 2n\sigma^2 \inf_{u  \in G} \|\gamma u+\epsilon\|_2^2$ when $X^\top X/n$ is near orthogonal, where $\gamma=0.5 \sigma^{-2}/n$.
  In comparison, under the noise free case $\sigma=0$ (and $\nabla L(\obeta)=0$), 
  the number of samples required in Gaussian random design is upper bounded by
  \[
  \rE_{\epsilon \sim N(0,I_{p \times p})} \inf_{u \in G} \|\gamma u+ \epsilon\|_2 
  \]
  for appropriate $\gamma$. The similarity of the two terms means that 
  it is expected that the error bound in oracle inequality and the number of
  samples required in Gaussian design are closely related.
\end{remark}

\subsection{Tangent Space Analysis} 
\label{sec:TRSC}

In some applications, the restricted strong convexity condition may not hold globally.
In this situation, one can further restrict the condition into a subspace $\cT$ of $\Omegabar$
call {\em tangent space} in the literature. 
We may regard tangent space as a generalization of the support set concept for sparse regression.
A more formal definition will be presented later in Section~\ref{sec:struct-tangent}.
In the current section, it can be motivated by considering the following decomposition of $G$:
\begin{equation}
G = \{u_0 + u_1 : u_0 \in G_0 \subset G, u_1 \in G_1\} ,
\label{eq:G-tangent-decomp}
\end{equation}
where $G_1$ is a convex set that contains zero.
Note that we can always take $G_0=G$ and $G_1=\{0\}$.
However, this is not an interesting decomposition.
This decomposition becomes useful when there exist $G_0$ and $G_1$ such that $G_0$ is small and $G_1$ is large. 
With this decomposition, we may define the tangent space as:
\[
\cT = \{\beta \in \Omegabar: \innerprod{u_1}{\beta}=0 \text{ for all } u_1 \in G_1 \} .
\]
For simple sparse regression with $\ell_1$ regularization, tangent space can be considered as the subspace
spanned by the nonzero coefficients of $\bbeta$ (that is, support of $\bbeta$).
Typically $\bbeta \in \cT$ (although this requirement is not essential).

With the above defined $\cT$, we may construct a tangent space dual certificate 
$Q_G^\cT$ given any $u_0 \in G_0$ as:
\begin{equation}
  Q_G^\cT= \bbeta + \Delta Q,
\quad \Delta Q = \arg\min_{\Delta \beta \in \cT} \left[ L(\bbeta+\Delta \beta) + \innerprod{u_0}{\Delta \beta} \right]  . \label{eq:dc-tangent-opt}
\end{equation}
Note that one may also define generalized dual tangent space certificate simply by working with
$\bar{L}_*(\beta)= \bar{L}(\beta) - \innerprod{\nabla \bar{L}(\bbeta)-\nabla L(\obeta)}{\beta-\bbeta}$ instead of $L(\beta)$.

The idea of tangent space analysis is to verify that the restricted dual certificate $Q_G^\cT$ is a dual certificate.
Note that to bound $D_L^s(\bbeta,Q_G^\cT)$,  we only need to assume restricted strong convexity inside $\cT$,
which is weaker than globally defined restricted convexity in Section~\ref{sec:RSC}.
The construction of $Q_G^\cT$ ensures that it satisfies the dual certificate definition in $\cT$ according to
Definition~\ref{def:primal-dual-certificate}, in that given any $\beta \in \cT: \innerprod{\nabla L(Q_G^\cT)-u_0}{\beta}=0$.
However, we still have to check that the condition (\ref{eq:dual-certificate}) 
holds for all $\beta \in \Omegabar$ to ensure that $Q_G=Q_G^\cT$
is a (globally defined) dual certificate. The sufficient condition is presented in the following proposition.
\begin{proposition}
Consider $Q_G^\cT$ in (\ref{eq:dc-tangent-opt}). If $-\nabla L(Q_G^\cT) -u_0 \in G_1$, then 
$Q_G=Q_G^\cT$ is a dual certificate that satisfies condition  (\ref{eq:dual-certificate}).
\end{proposition}
Technically speaking, the tangent space dual certificate analysis is a generalization of the irrepresentable condition 
for $\ell_1$ support recovery \cite{ZhaoYu06}. However, we are interested in oracle inequality rather than support recovery, and
in such context the analysis presented in this section generalizes those of \cite{CandesPlan09,CandesPlan11}.

\begin{definition}[Restricted Strong Convexity in Tangent Space]  
Given a subspace $\cT$ that contains $\bbeta$,
we define the following quantity which we refer to as tangent space restricted strong convexity (TRSC) constant:
\[
 \gamma_L^\cT(\bbeta;r,G,\|\cdot\|)= \inf \left\{ \frac{D_L^s(\beta,\bbeta)}{\|\bbeta-\beta\|^2} : 
\|\beta-\bbeta\|\leq r; \beta-\bbeta \in \cT; 
D_L^s(\beta,\bbeta)+ \innerprod{u_0+\nabla L(\bbeta)}{\beta-\bbeta} \leq 0 \right\} ,
 \]
where $\|\cdot\|$ is a norm, $r>0$ and $G \subset \partial R(\bbeta)$.
\end{definition}

\begin{theorem}[Dual Certificate Error Bound in Tangent Space]
Let $\|\cdot\|_D$ and $\|\cdot\|$ be dual norms, and consider convex $G \subset \partial R(\bbeta)$
with the decomposition (\ref{eq:G-tangent-decomp}).
If $\inf_{u \in G} \|u+\nabla L(\bbeta)\|_D < r \cdot \gamma_L^\cT(\bbeta;r,G,\|\cdot\|)$ for some $r>0$, then
\[
D_L^s(\bbeta,Q_G^\cT) \leq (\gamma_L^\cT(\bbeta;r,G,\|\cdot\|))^{-1}\|u_0 + P_\cT \nabla L(\bbeta)\|_D^2 ,
\]
where $Q_G^\cT$ is given by  (\ref{eq:dc-tangent-opt}).
\label{thm:dual_certificate-tangent-error}
\end{theorem}

If the condition $\inf_{u \in G} \|u+\nabla L(\bbeta)\|_D < r \cdot \gamma_L^\cT(\bbeta;r,G,\|\cdot\|)$ holds for some $r>0$, 
then
Theorem~\ref{thm:dual_certificate-tangent-error} implies that (\ref{eq:dc-tangent-opt}) has a finite solution. 
However, the bound using Theorem~\ref{thm:dual_certificate-tangent-error} may not be the sharpest possible. 
For specific problems, better bounds may be obtained using more refined estimates
(for example, in \cite{HsKaTz11-robust}).
If $Q_G^\cT$ is a globally defined dual certificate in that (\ref{eq:dual-certificate}) holds, then we immediately obtain 
results analogous to Corollary~\ref{cor:recovery-global-dc-error} and Corollary~\ref{cor:recovery-global-dc-quadratic}. 

Let $\bbeta_*$ be the target parameter in the sense that $\nabla L(\bbeta_*)$ is small.
If we want to apply Theorem~\ref{thm:dual_certificate-oracle}
in tangent space analysis, it may be convenient to consider the following choice of $\obeta$
instead of setting $\obeta$ to be the target $\bbeta_*$:
\begin{equation}
\obeta = \bbeta_* + \Delta \obeta , 
\qquad \Delta \obeta = \arg\min_{\Delta \beta \in \cT} L(\bbeta_* + \Delta \beta) .
\label{eq:target-opt}
\end{equation}
The advantage of this choice is that $\obeta$ is close to the target $\bbeta_*$, and thus $\nabla L(\obeta)$ is small.
Moreover,$\innerprod{\nabla L(\obeta)}{\beta}=0$ for all $\beta \in \cT$, which is convenient since it means
$\innerprod{\nabla \bar{L}_*(\bbeta)}{\beta}=0$ for all $\beta \in \cT$ with
$\bar{L}_*(\beta)= \bar{L}(\beta) - (\nabla \bar{L}(\bbeta)-\nabla L(\obeta))^\top (\beta-\bbeta)$.

For quadratic loss of (\ref{eq:quadratic-loss}), we have an analogy of Corollary~\ref{cor:recovery-global-dc-quadratic}.
Since $\innerprod{\cdot}{\cdot}$ becomes an inner product in a Hilbert space with $\Omegabar=\Omegabar^*$, 
we may further define the
orthogonal projection to $\cT$ as $P_\cT$ and to its orthogonal complements $\cT^\perp$ as $P_\cT^\perp$.
It is clear that in this case we also have $G_1 \subset \cT^\perp$.
\begin{corollary} \label{cor:recovery-tangent-dc-quadratic}
    Assume that $L(\beta)$ is a quadratic loss as in (\ref{eq:quadratic-loss}). 
    Consider convex $G \subset \partial R(\bbeta)$ with decomposition in (\ref{eq:G-tangent-decomp}).
    Consider $\obeta \in \Omegabar$ such that $2H\obeta-z=\tilde{a}+\tilde{b}$ with $\tilde{a} \in \cT$
    and $\tilde{b} \in \cT^\perp$. 
    Assume $H_\cT$, the restriction of $H$ to $\cT$, is invertible.
    If $u_0 \in \cT$, then let
    \[
    \Delta Q = - 0.5 H_\cT^{-1} (u_0+\tilde{a}) =
    \arg\min_{\Delta \beta \in \cT} \left[ \innerprod{H \Delta \beta}{\Delta \beta} + \innerprod{u_0+\tilde{a}}{\Delta \beta} \right] .
    \]
    If $P_\cT^\perp H H_\cT^{-1} u_0 - \tilde{b} \in G_1$, then
    \[
    D_L(\hbeta,\obeta)+ [R(\hbeta)- R_G(\hbeta) ] \leq 
    D_L(\bbeta,\obeta)+ 0.25 \innerprod{u_0+\tilde{a}}{H_\cT^{-1} (u_0+\tilde{a})} .
   \]
\end{corollary}
\begin{proof}
  Let $Q_G=\bbeta + \Delta Q$,
  then $Q_G$ is a generalized dual certificate that satisfies condition  (\ref{eq:dual-certificate-2}) with
  $\bar{L}=L$.
  This is because 
  \begin{align*}
    -\nabla \bar{L}_*(Q_G) -u_0 =& -2 H \Delta Q- \tilde{a}-\tilde{b} - u_0 \\
    =& -2 H H_\cT^{-1}(u_0+\tilde{a})- \tilde{a}-\tilde{b} - u_0 \\
    =& - 2 P_\cT H H_\cT^{-1}(u_0+\tilde{a})- 2 P_\cT^\perp H H_\cT^{-1} (u_0 + \tilde{a}) - \tilde{a} - \tilde{b} - u_0 \\
    =&  P_\cT^\perp H H_\cT^{-1} (u_0+\tilde{a})  - \tilde{b}  \in G_1 .
\end{align*}
We thus have
\[
D_L(\hbeta,\obeta)+ [R(\hbeta)- R_G(\hbeta) ] \leq 
D_L(\bbeta,\obeta)+  D_L(\bbeta,Q_G) .
\]
Since $D_L(\bbeta,Q_G)= \innerprod{H \Delta Q}{\Delta Q}=0.25 \innerprod{u_0+\tilde{a}}{H_\cT^{-1} (u_0+\tilde{a})}$,
the desired bound follows.
\end{proof}

If $\obeta$ is given by (\ref{eq:target-opt}), then $\tilde{a}=0$, and Corollary~\ref{cor:recovery-tangent-dc-quadratic}
can be further simplified.

\section{Structured $\ell_1$ regularizer}
\label{sec:struct-L1}
This section introduces a generalization of $\ell_1$ regularization for which the calculations 
in the dual certificate analysis can be relatively easily performed. It should be noted that
the general theory of dual certificate developed earlier can be applied to other regularizers that
may not have the structured form presented here.

Recall that $\Omegabar$ is a Banach space containing $\Omega$, $\Omegabar^*$ is its dual, 
and $\innerprod{u}{\beta}$ denotes $u(\beta)$ for linear functionals $u\in \Omegabar^*$. 
Let $E_0$ be either a Euclidean (thus $\ell_1$) space of a fixed dimension or a countably infinite dimensional $\ell_1$ space. We write any $E_0$-valued quantity as $a=(a_1,a_2,\ldots)^\top$ and 
bounded linear functionals on $E_0$ as $w^\top a = \sum_j w_ja_j=\innerprod{w}{a}$, 
with $w=(w_1,w_2,\ldots)^\top \in\ell_\infty$. 
Let $\scrM$ be the space of all bounded linear maps from $\Omegabar$ to $E_0$. 

Let $\scrA$ be a class of linear mappings in $\scrM$. We may define a regularizer as follows:
\begin{equation}
R(\beta) = \|\beta\|_\scrA, \qquad \|\beta\|_\scrA= \sup_{A\in\scrA}\| A\beta\|_1. 
\label{eq:struct-L1}
\end{equation}
As a maximum of seminorms, the regularizer $\|\beta\|_\scrA$
is clearly a seminorm in $\{\beta: R(\beta)<\infty\}$. 
The choice of $\scrA$ is quite flexible. 
We allow $R(\cdot)$ to have a nontrivial kernel $\ker(R)=\cap_{A\in\scrA}\ker(A)$. 
Given the $\scrA$-norm $\|\cdot\|_\scrA$ on $\Omegabar$, we may define its dual norm on $\Omegabar^*$ as
\[
\|u\|_\scrAD = \sup \{ \innerprod{u}{\beta}: \|\beta\|_\scrA \leq 1\} .
\]
Since $\|\beta\|_\scrA$ may take zero-value even if $\beta \neq 0$;
this means that $\|u\|_\scrAD$ may take infinite value, which we will allow in the following discussions.

We call the class of regularizers defined in (\ref{eq:struct-L1}) structured-$\ell_1$ (or structured-Lasso) regularizers. 
This class of regularizers contain enough structure so that dual certificate analysis can be carried out in generality. 
In the following, we shall discuss various properties of structured $\ell_1$ regularizer
by generalizing the corresponding concepts of $\ell_1$ regularizer for sparse regression. 
This regularizer obviously includes vector $\ell_1$ penalty as a special case.
In addition, we give two more structured regularization examples to illustrate the general applicability
of this regularizer.

\begin{example}\label{example:group-lasso}  
Group $\ell_1$ penalty: Let $E_j$ be fixed Euclidean spaces, 
$X_j:\Omegabar\to E_j$ be fixed linear maps, $\lam_j$ be fixed positive numbers, 
and $\scrA = \big\{(v_1^\top X_1,v_2^\top X_2,\ldots)^\top: 
v_j\in E_j, \|v_j\|_2\le \lam_j\big\}$. Then, 
\bes
R(\beta) = \sup_{A\in \scrA} \|A\beta\|_1 = \hbox{$\sum_j \lam_j\|X_j\beta\|_2$.}
\ees
\end{example}

\begin{example}\label{example:weighted-nuc} 
Nuclear penalty: $\Omegabar$ contains matrices of a fixed dimension.
Let $s_j(\beta)\ge s_{j+1}(\beta)$ denote the singular values of matrix $\beta$ and 
$\scrA = \big\{A: A\beta = (w_j(U^\top \beta V)_{jj}, j\ge 1), 
U^\top U =I_r, V^\top V =I_r, r\ge 0, 0\le w_j\le \lam\big\}$. Then, 
the nuclear norm (or trace-norm) penalty for matrix $\beta$ is 
\bes
R(\beta) = \sup_{A\in \scrA} \|A\beta\|_1 = \lam \sum_j s_j(\beta). 
\ees
\end{example}

\subsection{Subdifferential}
We characterize the subdifferential of $R(\beta)$ by studying the maximum property of $\scrA$. 
A set $\scrA$ is the largest class to generate (\ref{eq:struct-L1}) if for any $A_0\in\scrM$, 
$\sup_{\beta\in\Omegabar}\{\|A_0\beta\|_1-R(\beta)\} = 0$ implies $A_0\in\scrA$. 
We also need to introduce additional notations.
\begin{definition}
Given any map $M\in\scrM$, define its dual map $M^*$ from $\ell_\infty$ to $\Omegabar^*$ 
as: $\forall w\in\ell_\infty$, $M^* w$ satisfies  $\innerprod{M^*w}{\beta} = w^\top(M\beta), \forall \beta \in \Omegabar$. 
Given any $w \in \ell_\infty$, define $w(\cdot)$ as a linear map from $\scrM \to \Omegabar^*$ 
as $w(M)=M^* w$.
We also denote by $\overline{w(\scrA)}$ the closure of $w(\scrA)$ in $\Omegabar^*$. 
\end{definition}

The purpose of this definition is to introduce $e \in \ell_\infty$ so that 
$R(\beta)$ can be written as 
\[
R(\beta)=\sup_{A\in\scrA}\innerprod{e(A)}{\beta} = \sup_{u \in e(\scrA)}\innerprod{u}{\beta} .
\]
In this regard, one only needs to specifiy $e(\scrA)$ although for various problems it is more convenient to
specify $\scrA$.
Using this simpler representation, we have 
the following result characterizes the sub-differentiable of structured $\ell_1$ regularizer.
\begin{proposition}\label{prop:struct-subdiff} 
Let $E_1=\{w=(w_1,w_2,\ldots)^\top \in \ell_\infty: |w_j|=1\ \forall\ j\}$ 
and $e=(1,1,...)\in E_1$. \\
(i) A set $\scrA$ is the largest class generating $R(\beta)$ iff the following conditions hold: 
(a) $w(\scrA)=e(\scrA)$ for all $w\in E_1$; 
(b) $\scrA$ is convex; 
(c) $\scrA = \cap_{w\in E_1}w^{-1}(\overline{e(\scrA)})$, 
where $w^{-1}$ is the set inverse function. \\
(ii) Suppose $\scrA$ satisfied condition (a) in part (i). Then, 
$R(\beta)=\sup_{A\in\scrA}\innerprod{e(A)}{\beta}$. \\
(iii) Suppose $\scrA$ satisfied conditions (a) and (b) in part (i). Then, for $R(\beta)<\infty$, 
\bes
\pa R(\beta) = \{u\in \overline{e(\scrA)}: A\in\scrA, \innerprod{u}{\beta} = R(\beta)\}.
\ees 
\end{proposition}

In what follows, we assume $\scrA$ satisfied conditions (a) and (b) in (i). For notational 
simplicity, we also assume $e(\scrA) = \overline{e(\scrA)}$, which holds in the 
finite-dimensional case for closed $\scrA$. This gives 
\bel{eq:struct-subdiff}
\pa R(\beta) = \{u\in e(\scrA): A\in\scrA, \innerprod{u}{\beta} = R(\beta)\}.
\eel
Condition (c) in part (i) is then nonessential as it allows 
permutation of elements in $A$. Condition (c) holds for the specified $\scrA$ in 
Example \ref{example:weighted-nuc} but not in Example \ref{example:group-lasso}. 

\begin{proof} We assume (a) since it is necessary for $\scrA$ to be maximal in part (i). 

(ii) Under (a), $\sup_{A\in\scrA}\innerprod{e(A)}{\beta}
= \sup_{w\in E_1,A\in\scrA}\innerprod{w(A)}{\beta}
=\sup_{A\in\scrA,w\in E_1}w^\top(A\beta)=R(\beta)$. 

(i) We assume (b) since it is necessary. It suffices to prove the equivalence 
between the following two conditions for each $A_0\in\scrM$: 
$\sup_{\beta\in\Omegabar}\{\|A_0\beta\|_1-R(\beta)\} = 0$ 
and $A_0 \in \cap_{w\in E_1}w^{-1}(\overline{e(\scrA)})$. 

Let $A_0\in \cap_{w\in E_1}w^{-1}(\overline{e(\scrA)})$. 
For any $\beta\in\Omegabar$, there exists $w_0\in E_1$ such that 
$\|A_0\beta\|_1=w_0^\top A_0\beta = \innerprod{w_0(A_0)}{\beta}$. Since 
$A_0\in w_0^{-1}(\overline{e(\scrA)})$, $w_0(A_0)$ is the weak 
limit of $e(A_k)$ for some $A_k\in A$. It follows that 
$\|A_0\beta\|_1=\innerprod{w_0(A_0)}{\beta} = \lim_k\innerprod{e(A_k)}{\beta}
= \lim_k e^\top A_k\beta\le R(\beta)$. Now, consider 
$A_0\not\in w_0^{-1}(\overline{e(\scrA)})$, so that $w_0(A_0)\not\in\overline{e(\scrA)}$. 
This implies the existence of $\beta\in\Omegabar$ 
with $\|A_0\beta\|_1\ge \innerprod{w_0(A_0)}{\beta} > 
\sup_{A\in\scrA}\innerprod{e(A)}{\beta} =R(\beta)$. 

(iii) If $R(\beta)= \innerprod{u}{\beta}$ with $u\in \overline{e(A)}$, then 
$R(b)-R(\beta) \ge \innerprod{u}{b} - \innerprod{u}{\beta} 
= \innerprod{u}{b-\beta}$ for all $b$, so that $u\in\pa R(\beta)$. 
Now, suppose $v\in\pa R(\beta)$, so that 
$R(b)-R(\beta)\ge \innerprod{v}{b-\beta}$ for all $b\in\Omegabar$. 
Since $R(b)$ is a seminorm, taking $b=t\beta$ yields $R(\beta)=\innerprod{v}{\beta}$. 
Moreover, $\innerprod{v}{b-\beta}\le R(b-\beta)$ implies $v\in\overline{e(A)}$. 
The proof is complete.  
\end{proof}

\subsection{Structured Sparsity}
\label{sec:struct-tangent}
An advantage of the structured $\ell_1$ regularizer, compared with a general seminorm, 
is to allow the following notion of {\em structured sparsity}. 
A vector $\bbeta$ is sparse in the structure $\scrA$ if 
\bel{eq:struct-sparse}
\exists W\in \scrA: \quad R(\bbeta) = \innerprod{e(W)}{\bbeta},\quad S = \supp(W\bbeta), 
\eel
for certain set $S$ of relatively small cardinality. This means a small 
structured $\ell_0$ ``norm'' $\|W\bbeta\|_0$. In 
Example \ref{example:weighted-nuc}, this means $\beta$ has low rank. 

Let $e_S$ be the 0-1 valued
 $\ell_\infty$ vector with 1 on $S$ and 0 elsewhere. 
If $A\in\scrA$ can be written as $A=(W_S^\top,B_{S^c}^\top)^\top$, 
then $\|A\bbeta\|_1=\|W_S\bbeta\|_1+\|B_{S^c}\bbeta\|_1\le R(\bbeta)$, which 
implies $\|B_{S^c}\bbeta\|_1=0$ by (\ref{eq:struct-sparse}).
By (\ref{eq:struct-subdiff}), 
$e(A)=e((W_S^\top,B_{S^c}^\top)^\top) =e_S(W) + e_{S^c}(B) \in \pa R(\bbeta)$. Thus, we may choose 
\bel{eq:structG}
G_{\scrB} = \big\{e_S(W) + e_{S^c}(B) , B_{S^c}\in \scrB \big\}
\subseteq \pa R(\bbeta)
\eel
for a certain class $\scrB \subseteq \{B_{S^c}: (W_S^\top,B_{S^c}^\top)^\top\in \scrA\}$. 

Now let $G=G_{\scrB}$.
Since members of $G$ can be written as $e_S(W)+e_{S^c}(B), B \in \scrB$, this gives a decomposition
of $G$ as in (\ref{eq:G-tangent-decomp}) with
$G_0=\{u_0\}=\{e_S(W)\}$ and $G_1=e_{S^c}(\scrB)$.

Since $B\bbeta=0$ for $B\in \scrB$, we have
\bes
R_G(\beta) = R(\bbeta) + \sup_{u\in G}\innerprod{u}{\beta-\bbeta}
= \innerprod{e_S(W)}{\beta} + \sup_{B\in\scrB}\innerprod{e_{S^c}(B)}{\beta}. 
\ees

Unless otherwise stated, we assume the following conditions on $\scrB$: 
(a) $w_{S^c}(\scrB)=e_{S^c}(\scrB)$ for all $w\in E_1$; 
(b) $\scrB$ is convex; (c) $e_{S^c}(\scrB)$ is closed in $\Omegabar^*$. 
This is always possible since they match the assumed conditions on $\scrA$. 
Under these conditions, Proposition \ref{prop:struct-subdiff} gives 
\[
\sup_{B\in\scrB}\innerprod{e_{S^c}(B)}{\beta}= \sup_{B\in\scrB}\|B\beta\|_1 = \|\beta\|_\scrB .
\]
It's dual norm can be defined on $\Omegabar^*$ as
\[
\|u\|_\scrBD=\sup \left\{\innerprod{u}{\beta} : \|\beta\|_\scrB \leq 1 \right\} .
\]
This leads to the following simplified expression:
\bel{eq:structR_G}
R_G(\beta) = R(\bbeta) + \sup_{u\in G}\innerprod{u}{\beta-\bbeta}
= \innerprod{e_S(W)}{\beta} + \|\beta\|_\scrB . 
\eel

Since $B\bbeta=0$ for all $B \in \scrB$, $\scrB$ may be used to represent
a generalization of the zero coefficients of $\bbeta$, while $W_S$ can be used to represent 
a generalization of the sign of $\bbeta$. 
The larger the class $\scrB$ is, the more zero-coefficients $\bbeta$ has (thus $\bbeta$ is sparser). 
One may always choose $\scrB=\emptyset$ when $\bbeta$ is not sparse.

\subsection{Tangent Space} 

Given a convex function $\phi(\beta)$ and a point $\bbeta\in\Omega$, 
$b\in\Omegabar$ is a primal tangent vector if $\phi(\bbeta + tb)$ 
is differentiable at $t=0$. This means the equality of the left- and right-derivatives 
of $\phi(\bbeta + tb)$ at $t=0$. If $\phi(\beta)$ is a seminorm and $\bbeta\neq 0$, 
$\phi(\bbeta+t\bbeta)=(1+t)\phi(\bbeta)$ for all $|t|<1$, so that $\bbeta$ is always 
a primal tangent vector at $\bbeta$.  
If $\innerprod{u}{b}< \innerprod{v}{b}$ for $\{u,v\}\in\pa \phi(\bbeta)$, then 
\bes
\{\phi(\bbeta)-\phi(\bbeta - tb)\}/(0-t) \le \innerprod{u}{b} 
< \innerprod{v}{b} \le \{\phi(\bbeta+tb)-\phi(\bbeta)\}/t,\ \forall t>0, 
\ees
so that $\phi(\bbeta + tb)$ cannot be differentiable at $t=0$. This motivates 
the following definition of the (primal) tangent space of a regularizer at a point $\bbeta$ 
and its dual complement. 

\begin{definition}\label{def:struct-tangent} 
Given a convex regularizer $R(\beta)$, a point $\bbeta\in\Omega$, 
and a class $G\subseteq\pa R(\bbeta)$, we define the corresponding tangent space as 
\bes
\calT  = \calT_G = \big\{b \in \Omegabar: \innerprod{u-v}{b}=0\ \forall u\in G, v\in G\big\} 
= \cap_{u,v\in G}\ker(u-v). 
\ees
The dual complement of $\calT$, denoted by $\calT^\perp$, is defined as 
\bes
\calT^\perp = \calT_G^\perp =\hbox{closure}\Big\{u: u\in \Omegabar^*, 
\innerprod{u}{b}=0 \text{ for all } b \in \calT \Big\}. 
\ees
When $\innerprod{\cdot}{\cdot}$ is an inner product, $\Omegabar=\Omegabar^*$ and 
$\calT^\perp$ is the orthogonal complement of $\calT$ in $\Omegabar$. 
\end{definition}

\begin{remark}\label{remark:proj}
Let $\calT$ be any closed subspace of $\Omegabar$. 
A map $P_{\calT}: \Omegabar\to\Omegabar$ is a projection to $\calT$ if $P_{\calT}\beta=\beta$ 
is equivalent to $\beta\in\cal T$. For such $P_{\calT}$, its dual $P_{\calT}^*:\Omegabar^*\to\Omegabar^*$, defined by $\innerprod{P_{\calT}^* v}{\beta} = \innerprod{v}{P_{\calT}\beta}$, 
is a projection from $\Omegabar^*\to P_{\calT}^*\Omegabar^*$. 
The image of $P_{\calT}^*$, $\calT^* = P_{\calT}^*\Omegabar^*$, is a dual 
of $\calT$. Since $P_{\calT}$ and $P_{\calT}^*$ are projections, 
$v- P_{\calT}^*v\in \calT^\perp$ for all $v\in\Omegabar^*$ and 
$\beta - P_{\calT}\beta \in (\calT^*)^\perp$ for all $\beta \in\Omegabar$. 
\end{remark} 
 
The above definition is general. For the structured $\ell_1$ penalty, we let $G$ be as in (\ref{eq:structG}), 
we obtain by (\ref{eq:struct-subdiff}) that $\bbeta\in \calT$. 
The default conditions on $\scrB$ implies $0\in \scrB$, so that 
\bes
\calT = \big\{\beta:  \innerprod{e_{S^c}(B)}{\beta} = 0\ \forall B\in\scrB \big\} = \cap_{B\in\scrB}\ker(B). 
\ees
Since $G_1=e_{S^c}(\scrB)$, this is consistent with the definition of Section~\ref{sec:TRSC}. 
The dual complement of $\calT$ is 
\bes
\calT^\perp = \hbox{ the closure of the linear span of }\{e_{S^c}(B): B\in\scrB \}. 
\ees 

\subsection{Interior Dual Certificate and Tangent Sparse Recovery Analysis}
Consider a structured $\ell_1$ regularizer, a sparse $\bbeta \in \Omega$, and a set 
$G_{\scrB} \subset\pa R(\bbeta)$ as in (\ref{eq:structG}). 
In the analysis of (\ref{eq:hbeta}) with structured $\ell_1$ regularizer, members of 
the following subclass of $G_{\scrB}$ often appear.

\begin{definition}[Interior Dual Certificate] \label{def:interior-dc}
Given $G_{\scrB}$ in  (\ref{eq:structG}),  
$v_0$ is an interior dual certificate if 
\bes
v_0\in G_{\scrB},\ \innerprod{v_0-e_S(W)}{\beta}\le \eta_\beta \|\beta\|_\scrB \ 
\hbox{ for some $0\le \eta_\beta < 1$ for all $\beta$.}
\ees
\end{definition} 

Note that in the above definition, we refer to the dual variable $v_0$ as a ``dual certificate'' to be consistent with the
literature. This should not be confused with the notation of primal dual certificate $Q_G$ defined earlier. 
A direct application of interior dual certificate is the following extension 
of sparse recovery theory to general structured $\ell_1$ regularization. 
Suppose we observe a map $X:\Omegabar \to V$ with a certain linear space 
$V$. Suppose there is no noise so that $X\bbeta_* = y$ and $\bbeta=\bbeta_*$ is sparse.
Then the $R(\beta)$ minimization method for the recovery of $\bbeta$ is 
\bel{eq:hbeta-sparse-recover}
\hbeta = \argmin\Big\{R(\beta): X\beta = y\Big\}. 
\eel
The following theorem provides sufficient conditions for the recovery of $\bbeta$ by $\hbeta$. 

\begin{theorem}\label{thm:struct-recover} 
Suppose $\bbeta$ is sparse in the sense of (\ref{eq:struct-sparse}). 
Let $G$ be as in (\ref{eq:structG}) and $\calT$ be as in Definition \ref{def:struct-tangent}. 
Let $V^*$ be the dual of $V$, $X^*: V^*\to \Omegabar^*$ the dual of $X$, 
$P_{\calT}$ a projection to $\calT$, $P_{\cal}^*$ the dual of $P_{\cal}$ to $\calT^*$, 
and $V_T=XP_{\calT}\Omegabar$. 
Suppose $(XP_{\calT})^*$, the dual of $XP_{\calT}$, is a bijection from $V_T^*$ to $T^*$ 
and $e_S(W)\in T^*$. 
Define $v_0 =  X^*((XP_{\calT})^*)^{-1}e_S(W)$. If $v_0$ is an interior dual certificate, then 
\bes
\hbeta = \bbeta \hbox{ is the unique solution of 
(\ref{eq:hbeta-sparse-recover}).} 
\ees
Moreover, $v_0$ is an interior dual certificate iff for all $\beta$, there exists $\eta_\beta <1$ such that
$\innerprod{v_0 - P_{\calT}^*v_0}{\beta} \leq \eta_\beta \sup_{B\in\scrB}\|B\beta\|_1$. 
\end{theorem}

In matrix completion, this matches the duel certificate 
condition for recovery of low rank $\bbeta$ by constrained minimization of the nuclear 
penalty \cite{CandesR09,Recht09}. 

\begin{proof} Suppose $v_0$ is an interior dual certificate 
of the form $v_0=e_S(W)+e_{S^c}(B_0)$. Then, for all $\beta$ such that
$X\beta = y = X\bbeta$, 
\bes
R(\beta) - R(\bbeta) &=&
R(\beta) - R(\bbeta) - \innerprod{((XP_{\calT})^*)^{-1}e_S(W)}{X(\beta-\bbeta)}\\
&=& R(\beta) - R(\bbeta) - \innerprod{v_0}{\beta-\bbeta} 
\cr &\ge & \sup_{u\in G_\scrB}\innerprod{u-v_0}{\beta-\bbeta}
\cr &=& \sup_{B\in\scrB}\|B\beta\|_1 - \innerprod{e_{S^c}(B_0)}{\beta-\bbeta}
\cr &\ge& (1-\eta_\beta)\sup_{B\in \scrB}\|B\beta\|_1 .
\ees
with $\eta_\beta<1$. The first equation uses $X\beta=X\bbeta$, and the second equation uses
the definition $v_0=X^*((XP_{\calT})^*)^{-1}e_S(W)$.

Since (\ref{eq:hbeta-sparse-recover}) is constrained to $X\beta = y = X\bbeta$, the above inequality means that
$\bbeta$ is a solution of (\ref{eq:hbeta-sparse-recover}). It remains to prove its 
uniqueness. Let $\beta$ be another solution of (\ref{eq:hbeta-sparse-recover}). 
Since $1-\eta_\beta>0$, if $R(\beta)=R(\bbeta)$, then the above inequality implies that
$\max_{B\in \scrB}\|B\beta\|_1=0$, so that $\beta\in \calT$. 
Since $\bbeta\in\calT$, $XP_{\calT}(\beta-\bbeta)=X(\beta-\bbeta)=0$. 
This implies $\beta-\bbeta=0$, since the invertibility of $(XP_{\calT})^*$ 
implies $\calT\cap \ker(XP_{\calT}) = \{0\}$. 
\end{proof}

When noise is present, we may 
employ the construction of Section~\ref{sec:TRSC}. 
For structured-$\ell_1$ regularizer, the analysis can be further simplified if we assume that 
there exists a target vector $\obeta$ having the following property:
\begin{equation}
\nabla L(\obeta) = \tilde{a} + \tilde{b},
\label{eq:target}
\end{equation}
with a small $\tilde{a}$, and $\tilde{b}$ satisfies the condition
\[
\tilde{\eta} = \|\tilde{b}\|_\scrBD < 1 .
\]
Recall that the dual norm $\|\cdot\|_\scrBD$ of $\|\cdot\|_\scrB$ is defined as
$\|\tilde{b}\|_\scrBD=\sup \left\{\innerprod{\tilde{b}}{\beta} : \|\beta\|_\scrB \leq 1 \right\}$.
The condition means that there exists $\tilde{B} \in \scrB$ such that
$\tilde{b}= \tilde{w}_{S^c}(\tilde{B})$ with $\|\tilde{w}\|_\infty \leq \tilde{\eta}$.

For such a target vector $\obeta$, we will further consider an interior subset $G \subset G_{\scrB}$ in (\ref{eq:structG}) with some $\eta \in [\tilde{\eta},1]$: 
\begin{equation}
G=\{e_S(W)+\eta e_{S^c}(B): B \in \scrB\}  \label{eq:structG-int} .
\end{equation}
It follows that 
\[
R(\beta)-R_G(\beta) \geq R_{G_\scrB}(\beta) - R_G(\beta) \geq \eta \|\beta\|_\scrB
\]
and
\bes
\sup_{u\in G}\innerprod{u+\nabla L(\obeta)}{\beta-\bbeta}
&=& \innerprod{e_S(W)+\tilde{a}}{\beta-\bbeta}
+ \sup_{u\in G}\innerprod{e_{S^c}(B)+\tilde{\eta} e_{S^c}(\tilde{B})}{\beta-\bbeta}
\cr &\ge& \innerprod{e_S(W)+\tilde{a}}{\beta-\bbeta}
+ (\eta- \tilde{\eta})\|\beta-\bbeta\|_\scrB .
\ees
This estimate can be directly used in the definition of RSC in  Corollary \ref{cor:recovery-global-dc-oracle}. 
One way to construct such a target vector $\obeta$ is using (\ref{eq:target-opt}).
In this case we may further assume that $\tilde{a}=0$ because 
$\innerprod{\nabla L(\obeta)}{\beta}=0$ for any $\beta \in \cT$. 
In general condition (\ref{eq:target}) is relatively easy to satisfy under the usual stochastic noise model
with a small $\bar{a}$ since $\nabla L(\obeta)$ is small.
In the special setting of Theorem~\ref{thm:struct-recover}, we have $\nabla L(\obeta)=0$ with $\obeta=\bbeta=\bbeta_*$.

For simplicity, in the following we will consider quadratic loss of the
form (\ref{eq:quadratic-loss}) and apply Corollary~\ref{cor:recovery-tangent-dc-quadratic}.
Consider $G$ in (\ref{eq:structG-int}), $\obeta$ in (\ref{eq:target}) with $\tilde{a} \in \cT$ ($\tilde{b} \in \cT^\perp$), 
and $Q_G^\cT$ 
defined as in (\ref{eq:dc-tangent-opt}) but with $L(\beta)$ replaced by 
$\bar{L}_*(\beta)= L(\beta) - (\nabla L(\bbeta)-\nabla L(\obeta))^\top (\beta-\bbeta)$, which can be equivalently written as
\[
  Q_G^\cT = \bbeta + \Delta Q, \quad
  \Delta Q= - 0.5 H_\cT^{-1} (e_S(W)+\tilde{a}) .
\]
This is consistent with the construction of Theorem~\ref{thm:struct-recover} in the sense that in the
noise-free case, we can let $H=X^\top X$ and $v_0=-2H \Delta Q= H H_\cT^{-1} e_S(W)$ with $\tilde{a}=0$.

We assume that the following condition holds for all $\beta$: 
\begin{equation}
\| P_\cT^\perp H H_\cT^{-1} (e_S(W)+\tilde{a})  - \tilde{b}\|_\scrBD \leq \eta , 
\label{eq:irrep-struct}
\end{equation}
which is consistent with the noise free interior dual certificate existence condition in Theorem~\ref{thm:struct-recover}
by setting $\eta_\beta=\eta$.
The condition is a direct generalization of the strong irrepresentable condition for $\ell_1$ regularization in 
\cite{ZhaoYu06} to structured $\ell_1$ regularization. 
Under this condition, $Q_G^\cT$ is a dual certificate that satisfies the generalized 
condition (\ref{eq:dual-certificate-2}) in
Definition~\ref{def:primal-dual-certificate-2} with $\bar{L}=L$ and $\delta=0$.
Corollary~\ref{cor:recovery-tangent-dc-quadratic} implies that
\[
D_L(\obeta,\hbeta)+ (1-\eta)  \|\hbeta\|_\scrB 
\leq D_L(\obeta,\bbeta) + 0.25 \innerprod{e_S(W)+\tilde{a}}{H_\cT^{-1} (e_S(W)+\tilde{a})} .
\]

\subsection{Recovery Analysis with Global Restricted Strong Convexity}

We can also employ the dual certificate construction of Section~\ref{sec:RSC}
with $G$ in (\ref{eq:structG-int}) and $\obeta=\bbeta$. 
Corollary~\ref{cor:recovery-global-dc-error} implies the following result:
\[
D_L(\bbeta,\hbeta) + (1-\eta)  \|\hbeta\|_\scrB 
\leq \gamma_L(\bbeta;r,G,\|\cdot\|)^{-1} \|e_S(W) +\tilde{a}\|_D^2 ,
\]
where $\gamma_L(\bbeta;r,G,\|\cdot\|)$ is lower bounded by
\[
\inf \left\{ \frac{D_L^s(\beta,\bbeta)}{\|\bbeta-\beta\|^2} : 
\|\beta-\bbeta\|\leq r; \;
D_L^s(\beta,\bbeta)+ (\eta-\tilde{\eta}) \|\beta-\bbeta\|_\scrB +
\innerprod{e_S(W) + \tilde{a}}{\beta-\bbeta} \leq 0 \right\} .
\]

We may also consider a more general $\obeta$ instead of assuming $\obeta=\bbeta$. 
For example, consider the definition of $\obeta$ in (\ref{eq:target-opt}), which implies that $\tilde{a}=0$
or simply let $\obeta=\bbeta_*$.
We can apply Corollary~\ref{cor:recovery-global-dc-quadratic} to the
quadratic loss function of (\ref{eq:quadratic-loss}). It implies
\begin{equation}
D_L(\obeta,\hbeta)+ (1-\eta)  \sup_{B \in \scrB} \|B \hbeta\|_1 
\leq D_L(\obeta,\bbeta) + 
(2\gamma_{\bar{L}_*}(\bbeta;\infty,G,\|\cdot\|))^{-1} \|\tilde{a}+e_S(W)\|_D^2 ,
\label{eq:recovery-global-dc-quadratic-struct}
\end{equation}
where $\gamma_{\bar{L}_*}(\bbeta;\infty,G,\|\cdot\|)$ is lower bounded by
\[
\inf \left\{ \frac{2\innerprod{H\beta}{\beta}}{\|\beta\|^2} : 
2\innerprod{H\beta}{\beta}
+ (\eta-\tilde{\eta}) \|\beta\|_\scrB +
\innerprod{e_S(W)+\tilde{a}}{\beta} \leq 0 \right\} .
\]

\subsection{Recovery Analysis with Gaussian Random Design}
\label{sec:gordon-struct}
We can also apply the results of Section~\ref{sec:gordon} 
by considering quadratic loss with Gaussian random design matrix in (\ref{eq:gaussian-design}).
We can use the following proposition
\begin{proposition}
  If $\tilde{\eta}<\eta$ and $\epsilon \sim N(0,I_{p \times p})$, then
  \begin{align*}
\rE_{\epsilon}^2 \inf_{u \in G; \gamma >0} \|\gamma(u+\nabla L(\obeta)) - \epsilon\|_2 
  \leq
  \inf_{\gamma>0} 
    \rE_{\epsilon} \inf_{B \in \scrB} \| \gamma (e_S(W) +\tilde{a} +(\eta-\tilde{\eta})e_S(B)) - \epsilon\|_2^2 
.
\end{align*}
\end{proposition}

Therefore we may apply Theorem~\ref{thm:recovery-gaussian}, which implies that 
given any $g,\delta \geq 0$ such that $g+\delta \leq n/\sqrt{n+1}$, with probability at least
\[
1 - \frac{1}{2}\exp \left(-\frac{1}{2} (n/\sqrt{n+1}-g-\delta)^2\right) ,
\]
we have either $\tilde{\eta} \geq \eta$, or
  \[
  \|X(\hbeta-\obeta)\|_2^2 + (1-\eta)  \|\hbeta\|_\scrB  \leq 
  \|X(\bbeta-\obeta)\|_2^2 + 
  (4\delta)^{-1} \|e_S(W) +\tilde{a}\|_2^2 ,
 \]
  or 
  \[
  g^2 < 
  \inf_{\gamma>0} 
    \rE_{\epsilon \sim N(0,I_{p \times p})} \inf_{B \in \scrB} \| \gamma (e_S(W) +\tilde{a} +(\eta-\tilde{\eta}) e_S(B)) - \epsilon\|_2^2 .
\]
\subsection{Parameter Estimation Bound}

Generally speaking, the technique of dual certificate allows us to directly obtain an oracle inequality
\begin{equation}
D_L(\hbeta,\obeta)+ (1-\eta) \|\hbeta\|_\scrB\leq \delta \label{eq:struct-oracle}
\end{equation}
for some $\delta >0$.
If $\delta$ is small (in such case, $\bbeta$ should be close to $\obeta$),
then we may also be interested in parameter estimation bound 
$\|\hbeta-\obeta\|$. In such case, additional estimates will be needed on top of the dual certificate theory of this paper.
This section demonstrate how to obtain such a bound from (\ref{eq:struct-oracle}). 

Although parameter estimation bounds can be obtained for general loss functions $L(\cdot)$,
they involve relatively complex notations.
In order to illustrate the main ideas while avoiding unnecessary complexity, in the following we 
will only consider the quadratic loss case, where $\innerprod{\cdot}{\cdot}$ is an inner product.
\begin{proposition} \label{prop:param-est}
  Assume that $L(\cdot)$ is the quadratic loss function given by (\ref{eq:quadratic-loss}). 
  Consider any subspace $\tcT$ that contains the tangent space $\cT$.
  Let $\delta' = \delta/(1-\eta) + \|P_\cT^\perp \obeta\|_\scrB$ with $\delta$ given by (\ref{eq:struct-oracle}).
  Define the correlation between $\tcT$ and $\tcT^\perp$ as:
  \[
  \cor(\tcT,\tcT^\perp)= \sup \left\{ |\innerprod{H P_\tcT \beta}{P_\tcT^\perp \beta}|/ \innerprod{H P_\tcT \beta}{P_\tcT \beta}^{1/2} : 
 \obeta + \beta \in \Omega,  \|\beta\|_\scrB \leq \delta' \right\} .
  \]
  Let $\Delta=\hbeta-\obeta$. Then, 
  $\|\Delta \|_\scrB \leq \delta'$, and
  \[
  \innerprod{H_\tcT P_\tcT \Delta}{P_\tcT \Delta}^{1/2} \leq \sqrt{(1-\eta) \delta'} + 2 \cor(\tcT,\tcT^\perp) .
  \]
\end{proposition}
\begin{proof}
We have
\begin{align*}
&\innerprod{H_\tcT P_\tcT \Delta}{P_\tcT \Delta}
+ 2  \innerprod{H P_\tcT \Delta}{P_\tcT^\perp \Delta} +
(1-\eta) \|\Delta \|_\scrB \\
\leq&\innerprod{H_\tcT P_\tcT \Delta}{P_\tcT \Delta}
+ 2  \innerprod{H P_\tcT \Delta}{P_\tcT^\perp \Delta} +
\innerprod{H_\tcT P_\tcT^\perp \Delta}{P_\tcT^\perp \Delta} +
(1-\eta) \|\hbeta \|_\scrB +
(1-\eta) \|\obeta \|_\scrB \\
=& D_L(\hbeta,\obeta) + (1-\eta) \|\hbeta\|_\scrB + (1-\eta) \|P_\cT^\perp \obeta\|_\scrB
\leq (1-\eta) \delta' ,
\end{align*}
where we have used the fact that $\|\obeta\|_\scrB=\|P_\cT^\perp \obeta\|_\scrB$.
This means that if we let $\beta=\Delta$, then we have
$\|\beta \|_\scrB \leq 1$, and $\obeta + \beta \in \Omega$.
Let $x^2=\innerprod{H_\tcT P_\tcT \Delta}{P_\tcT \Delta}$, we have
$\innerprod{H P_\tcT \Delta}{P_\tcT^\perp \Delta} =x \innerprod{H P_\tcT \beta}{P_\tcT^\perp \beta}/ \innerprod{H P_\tcT \beta}{P_\tcT \beta}^{1/2}$. It follows that
\[
x^2- 2 x \cor(\tcT,\tcT^\perp) \leq (1-\eta) \delta' .
\]
Solving for $x$ leads to the desired bound. 
\end{proof}

Clearly, we can have a cruder estimate:
\[
\cor(\tcT,\tcT^\perp)
\leq \sup \{ \innerprod{H P_\tcT^\perp \beta}{P_\tcT^\perp \beta}^{1/2} : \obeta + \beta \in \Omega, \|\beta\|_\scrB \leq \delta'\} .
\]
The bound in Proposition~\ref{prop:param-est} is useful when $H$ is invertible on $\tcT$:
\[
\innerprod{H \beta}{\beta} \geq \gamma_{\tcT} \innerprod{\beta}{\beta} \quad \forall \beta \in \tcT ,
\]
which leads to a bound on $\|P_\tcT \Delta\|_2$.
Although one may simply choose $\tcT=\cT$, the resulting bound may be suboptimal, as we shall see later on.
Therefore it can be beneficial to choose a larger $\tcT$.
Examples of this result will be presented in Section~\ref{sec:examples}.

\section{Examples}
\label{sec:examples}
We will present a few examples to illustrate the analysis as well as concrete substantiations of the relatively
abstract notations we have used so far.

\subsection{Group $\ell_1$ Least Squares Regression}

We assume that $\Omegabar = \Real^p$, and consider the model
\[
Y= X \bbeta_* + \epsilon 
\]
with the least squares loss function (\ref{eq:gaussian-design}).
This corresponds to the quadratic loss (\ref{eq:quadratic-loss}) with $H=X^\top X$ and $z=2X^\top Y$.
The inner product is Euclidean: $\innerprod{u}{{b}}=u^\top {b}$.

Now, we assume that $p=q m$, and the variables $\{1,\ldots,p\}$ are divided into $q$ non-overlapping blocks
$\G_1,\ldots,\G_{q} \subset \{1,\ldots,p\}$ of size $m$ each.
One method to take advantage of the group structure is to use the group Lasso method \cite{YuanLin06} with
\begin{equation}
R(\beta)= \lambda \|\beta\|_{\G,1} , \qquad  \|\beta\|_{\G,1}=\sum_{j=1}^q \|\beta_{\G_j}\|_2 .
\label{eq:group-L1}
\end{equation}
Its dual norm is
\[
\|\beta\|_{\G,\infty} = \max_j \|\beta_{\G_j}\|_2 .
\]
Group $\ell_1$ regularization 
includes the standard $\ell_1$ regularization as a special case, where we choose $m=1$, $q=p$, and $\G_j=\{j\}$.

Group-$\ell_1$ regularizer is a special case of (\ref{eq:struct-L1}), where we have
\[
\scrA=\{A=(a_j): A \beta=(a_j^\top \beta)_{j=1,\ldots,q}: a_j \in \Real^p, \|a_j\|_2 \leq \lambda, \; \supp(a_j) \subset \G_j\} .
\]

For a group sparse $\bbeta$, its group support is the smallest $S \subset \{1,\ldots,q\}$
such that $\supp(\bbeta) \subset \GS=\cup_{k  \in S} \G_k$.
We may define $\sgn_\G(\bbeta_{\G_j})$ to be $\sgn_\G(\bbeta_{\G_j})=\bbeta_{\G_j}/\|\bbeta_{\G_j}\|_2$ 
when $j \in S$, and $\sgn_\G(\bbeta_{\G_j})=0$ when $j \notin S$.

Using notations in Section~\ref{sec:struct-L1}, we may take $W=(\lambda \sgn_\G(\bbeta_{\G_j}))_{j=1,\ldots,q}$, and
$\scrB=\{B=(b_j) \in \scrA: b_j=0 \text{ for all } j \in S \}$
in (\ref{eq:structG}).
In fact, our computation does not directly depend on $W$ and $\scrB$. Instead, we may simply
specify 
\[
e_S(W)=\lambda \sgn_\G(\bbeta) \quad \text{and} \quad
e_{S^c}(\scrB)=\{b \in \Real^p: \|b\|_{\G,\infty} \leq \lambda; \supp(b) \subset \GS^c\} ,
\]
and
\[
\|\beta\|_\scrB = \lambda \|\beta_{\GS^c}\|_{\G,1} \qquad
\|b_{\GS^c}\|_\scrBD =  \|b_{\GS^c}\|_{\G,\infty}/\lambda .
\]

This means that we may take $G$ in (\ref{eq:structG-int}) as
$G=\{u; \; u_\GS=\lambda \sgn_\G(\bbeta) \quad \& \quad \|u_{\GS^c}\|_{\G,\infty} \leq \eta \lambda\}$ for some $0 \leq \eta \leq 1$,
which implies that $R(\beta)-R_G(\beta) \geq (1-\eta) \|\beta_{\GS^c}\|_{\G,1}$. 
The tangent space is $\cT=\{u: \supp(u) \in \GS\}$.

We further consider target $\obeta$ that satisfies (\ref{eq:target}), which we can rewrite as
\[
2 X^\top (X(\obeta -\bbeta_*)- \epsilon) = \tilde{a} + \tilde{b} ,
\]
where $\supp(\tilde{b}) \subset \GS^c$, and $\|\tilde{b}\|_{\G,\infty} = \tilde{\eta} \lambda$.
We assume that $\|\tilde{a}\|_2$ is small.
Note that we may choose $\lambda$ sufficiently large so that $\tilde{\eta}$ can be arbitrarily close to $0$.
In particular, we may choose $\lambda \geq \|\tilde{b}\|_{\G,\infty}/\eta$ so that $\tilde{\eta} \leq \eta <1$.
We are specially interested in the case of $\tilde{a}=0$, which can be achieved with the construction in (\ref{eq:target-opt}).

\subsubsection*{Global Restricted Eigenvalue Analysis}
Assume that $\lambda \geq \|\tilde{b}\|_{\G,\infty}/\eta$, and let
$\tilde{\eta}=\|\tilde{b}\|_{\G,\infty}/\lambda$. We have $\tilde{\eta} \leq \eta$. 
Therefore in order to apply (\ref{eq:recovery-global-dc-quadratic-struct}), we may define restricted eigenvalue as
\[
\gamma =\inf \left\{2\|X \Delta \beta\|_2^2/\|\Delta \beta\|^2: 2\|X \Delta \beta\|_2^2 + 
 \Delta \beta^\top (\lambda \sgn_\G(\bbeta) +\tilde{a}) + (\eta-\tilde{\eta}) \lambda \|\Delta \beta_{\GS^c}\|_{\G,1} \leq 0 \right\} .
\]
We then obtain from (\ref{eq:recovery-global-dc-quadratic-struct})
\[
\| X (\hbeta-\obeta)\|_2^2+ (1-\eta) \lambda \|\hbeta_{\GS^c}\|_{\G,1} \leq 
\| X (\bbeta-\obeta)\|_2^2 +
(2\gamma)^{-1} \|\tilde{a}+\lambda \sgn_\G(\bbeta)\|_D^2 .
\]
If we choose $\tilde{a}=0$, and let $\|\cdot\|=\|\cdot\|_{\G,1}$ with $\|\cdot\|_D=\|\cdot\|_{\G,\infty}$,
then
\[
\| X (\hbeta-\obeta)\|_2^2+ (1-\eta) \lambda \|\hbeta_{\GS^c}\|_{\G,1} \leq 
\| X (\bbeta-\obeta)\|_2^2 + 
\frac{\lambda^2|S|}{4 \bar{\gamma}} ,
\]
with
\[
\bar{\gamma} =
\inf \left\{\|X \Delta \beta\|_2^2/(\|\Delta \beta\|_{\G,1}^2/|S|): 
\Delta \beta_\GS^\top \sgn_\G(\bbeta) + (\eta-\tilde{\eta}) \|\Delta \beta_{\GS^c}\|_{\G,1} \leq 0 \right\} .
\]
The result is meaningful as long as $\bar{\gamma}>0$. 
Even for the standard $\ell_1$ regularizer, this condition is weaker than previous
restricted eigenvalue conditions in the literature. In particular it is weaker than 
the compatibility condition of \cite{vandeGeerB09} (which is the weakest condition in the earlier literature), 
that requires
\[
\inf \left\{\|X \Delta \beta\|_2^2/(\|\Delta \beta\|_{1}^2/|S|): 
(1-\tilde{\eta})  \|\Delta \beta_{\GS^c}\|_{\G,1} \leq 
(1+\tilde{\eta}) \|\Delta \beta_\GS\|_{\G,1} \right\} 
> 0. 
\]
Our result replaces $\|\Delta \beta_\GS\|_{\G,1}$ by $-\Delta \beta_\GS^\top \sgn_\G(\bbeta)$, 
which is a useful improvement because the former can be significantly larger than the latter. 
For $\ell_1$ analysis, the use of $\sgn(\bbeta)$ has appeared in various studies such as
\cite{Wainwright09,ChRePaWi10,CandesPlan09,CandesPlan11}. In fact, the calculation for 
Gaussian random design, which we shall perform next, depends on $\sgn(\bbeta)$ and $\sgn_\G(\bbeta)$.

\subsubsection*{Gaussian Random Design}

Assume that $X$ is Gaussian random design matrix in (\ref{eq:gaussian-design}), then
we can apply the analysis in Section~\ref{sec:gordon-struct}.
We will first consider the standard $\ell_1$ regularizer with $m=1$, which requires the following estimate.
\begin{proposition}
  Consider standard $\ell_1$ regularization with single element groups.
  If $\tilde{\eta} < \eta$ and $p \geq 2|S|$, we have
  \begin{align*}
    &\inf_{\gamma>0} 
    \rE_{\epsilon \sim N(0,I_{p \times p})} \inf_{\|b\|_\infty \leq 1} 
    \| \gamma (\lambda \sgn(\bbeta)+\tilde{a}) + \gamma (\eta-\tilde{\eta})\lambda b_{S^c} - \epsilon\|_2^2  \\
   \leq&
   2 |S| + \frac{2\ln (p/|S|-1)}{(\eta-\tilde{\eta})^2}
  \|\sgn(\bbeta)+\tilde{a}/\lambda\|_2^2 .
 \end{align*}
\end{proposition}
\begin{proof} 
  Given $\gamma>0$, and let $t=\gamma(\eta-\tilde{\eta}) \lambda$, we have
  \[
  \rE_{\epsilon \sim N(0,I_{p \times p})}\inf_{\|b\|_\infty \leq 1} 
  \| \gamma (\lambda \sgn(\bbeta)+\tilde{a})+ \gamma (\eta-\tilde{\eta}) \lambda b_{S^c} - \epsilon\|_2^2 
  \leq a_0 + a_1 ,
 \]
 where 
 \[
 a_0 = \rE_{\epsilon \sim N(0,I_{p \times p})}
 \| \gamma (\sgn(\bbeta)+\tilde{a}) + \epsilon_S\|_2^2 = 
 |S| + \gamma^2 \|\lambda \sgn(\bbeta)+\tilde{a}\|_2^2 ,
 \]
 and 
 \begin{align*}
 a_1 =& \rE_{\epsilon \sim N(0,I_{p\times p})} \inf_{\|b\|_\infty \leq t}\| b_{S^c} - \epsilon_{S^c}\|_2^2 \\
 =& (p-|S|) \rE_{\epsilon \sim N(0,1)} (|\epsilon|-t)_+^2 \\
 =& (p-|S|)  \int_{x=0}^\infty \frac{2}{\sqrt{2\pi}} x^2 \exp(-(x+t)^2/2) d x \\
\leq& (p-|S|)  \int_{x=0}^\infty \frac{2}{\sqrt{2\pi}} x^2 \exp(-(x^2+t^2)/2) d x 
\leq (p-|S|) e^{-t^2/2} .
\end{align*}
By setting $t=\sqrt{2\ln ((p/|S|-1)}$ and $\gamma=\sqrt{2\ln((p/|S|-1))}/(\eta-\tilde{\eta})\lambda$, 
we have $a_1 \leq |S|$. This gives the bound.
\end{proof}

For the standard $\ell_1$ regularization ($m=1$), we obtain the following bound if $p \geq 2|S|$:
given any $\eta \in (0,1]$, $g,\delta \geq 0$ such that $g+\delta \leq n/\sqrt{n+1}$, with probability at least
\[
1 - \frac{1}{2} \exp \left(-\frac{1}{2} (n/\sqrt{n+1}-g-\delta)^2\right) ,
\]
we have either $\lambda \leq \|\tilde{b}\|_\infty/ \eta$, or
  \[
  \|X(\hbeta-\obeta)\|_2^2 + (1-\eta)  \lambda \|\hbeta_{S^c}\|_1  \leq 
  \|X(\bbeta-\obeta)\|_2^2 + 
  (4\delta^2)^{-1} \|\lambda\sgn(\bbeta) +\tilde{a}\|_2^2 ,
 \]
 or 
  \[
  g^2 < 
   2 |S| + \frac{2\ln (p/|S| -1)}{(\eta-\|\tilde{b}\|_\infty/\lambda)^2}
  \|\sgn(\bbeta)+\tilde{a}/\lambda\|_2^2 .
\]

Note that in the noise-free case of $\tilde{a}=\tilde{b}=0$, this shows that exact recovery can be achieved
with large probability when $n > 2|S| (1+ \ln(p/|S|-1))$, and this sample complex result is a rather sharp.
More generally for $m>1$, we have a similar bound with worse constants as follows.
\begin{proposition}
  If $\tilde{\eta} < \eta$ and $p \geq 2m |S|$, we have
  \begin{align*}
    &\inf_{\gamma>0} 
    \rE_{\epsilon \sim N(0,I_{p \times p})} \inf_{\|b\|_{\G,\infty} \leq 1} 
    \| \gamma (\lambda \sgn_\G(\bbeta)+\tilde{a}) + \gamma (\eta-\tilde{\eta})\lambda b_{S^c} - \epsilon\|_2^2  \\
   \leq&
   |S|(m+1) + \frac{(\sqrt{2\ln (q/|S|-1)}+\sqrt{m})^2}{(\eta-\tilde{\eta})^2}
  \|\sgn_\G(\bbeta)+\tilde{a}/\lambda\|_2^2 .
 \end{align*}
\end{proposition}
\begin{proof}
  Given $\gamma>0$, and let $t=\gamma (\eta-\tilde{\eta}) \lambda$.
  Let $\chi$ be a $\chi$-distributed random variable of degree $m$,  with $\lambda_m$ being its expectation as defined in
  Theorem~\ref{thm:gordon}. Since $\chi$ is the singular value of a $1 \times m$ Gaussian matrix, 
  similar to Theorem~\ref{thm:gordon}, we can apply the Gaussian concentration
  bound \cite{Pisier85} to obtain for all $\delta >0$:
  \[
  \rP \left[ \chi \geq \lambda_m + \delta \right] \leq 0.5 \exp \left(-\delta^2/2 \right) .
  \]
  Now we assume $t\geq \lambda_m$, and
 \[
  \rE_{\epsilon \sim N(0,I_{p \times p})}\inf_{\|b\|_{\G,\infty} \leq 1} 
  \| \gamma (\lambda \sgn_\G(\bbeta)+\tilde{a})+ \gamma (\eta-\tilde{\eta}) \lambda b_{S^c} - \epsilon\|_2^2 
  \leq a_0 + a_1 ,
 \]
 where 
 \[
 a_0 = \rE_{\epsilon \sim N(0,I_{p \times p})}
 \| \gamma (\sgn_\G(\bbeta)+\tilde{a}) + \epsilon_\GS\|_2^2 = 
 m |S| + \gamma^2 \|\lambda \sgn(\bbeta)+\tilde{a}\|_2^2 ,
 \]
 and 
 \begin{align*}
 a_1 =& \rE_{\epsilon \sim N(0,I_{p\times p})} \inf_{\|b\|_{\G,\infty} \leq t}\| b_{\GS^c} - \epsilon_{\GS^c}\|_2^2 \\
 =& (q-|S|) \rE_{\epsilon \sim N(0,I_{m \times m})} (\|\epsilon\|_2-t)_+^2 \\
 =& - (q-|S|)  \int_{x=0}^\infty  x^2 d P(\chi \geq x+t) \\
\leq& 2 (q-|S|)  \int_{x=0}^\infty  x P(\chi \geq x+t) d x \\
\leq& 2 (q-|S|)  \int_{x=0}^\infty  0.5 x \exp(-(x+t- \lambda_m)^2/2) d x \\
\leq&  (q-|S|)  \exp(-(t-\lambda_m)^2/2) \int_{x=0}^\infty  x \exp(-x^2/2) d x \\
=&  (q-|S|)  \exp(-(t-\lambda_m)^2/2) .
\end{align*}
By setting $t=\lambda_m + \sqrt{2\ln (q/|S|-1)}$ and $\gamma=t/(\eta-\tilde{\eta})\lambda$, 
we have $a_1 \leq |S|$. This gives the desired bound using the estimate $\lambda_m \leq \sqrt{m}$.
\end{proof}

We obtain the following bound for group-Lasso with $m>1$ when $q \geq 2|S|$:
given any $\eta \in (0,1]$, $g,\delta \geq 0$ such that $g+\delta \leq n/\sqrt{n+1}$, with probability at least
\[
1 -\frac{1}{2} \exp \left(-\frac{1}{2} (n/\sqrt{n+1}-g-\delta)^2\right) ,
\]
we have either $\lambda \leq \|\tilde{b}\|_{\G,\infty}/ \eta$, or
\[
\|X(\hbeta-\obeta)\|_2^2 + (1-\eta)  \lambda \|\hbeta_{\GS^c}\|_{\G,1}  \leq 
\|X(\bbeta-\obeta)\|_2^2 + 
(4\delta^2)^{-1} \|\lambda\sgn_\G(\bbeta) +\tilde{a}\|_2^2 ,
\]
or 
\[
g^2 < 
|S|(m+1) + \frac{(\sqrt{2\ln (q/|S| -1)}+\sqrt{m})^2}{(\eta-\|\tilde{b}\|_{\G,\infty}/\lambda)^2}
\|\sgn_\G(\bbeta)+\tilde{a}/\lambda\|_2^2 .
\]

Note that in the noise-free case of $\tilde{a}=\tilde{b}=0$, this shows that exact recovery can be achieved
with large probability when $n > |S|(m+1) + |S|(\sqrt{2\ln (q/|S| -1)}+\sqrt{m})^2=O(|S| (m+ \ln (q/|S|)))$.

If we consider the scenario that noise $\epsilon \sim N(0,\sigma^2 I_{n \times n})$ is Gaussian, then
we may set $\lambda$ to be at the order $\sigma \sqrt{n (m+\ln (q/|S|))}$, and with large probability,
we have $\lambda > \|b\|_{\G,\infty}/\eta$, with a nonzero $\tilde{a}$ such that
$\|\tilde{a}\|_2^2= O(|S| \lambda^2)$. This gives the following error bound
with $\delta$ chosen at order $\sqrt{n}$:
  \[
  \|X(\hbeta-\obeta)\|_2^2 + (1-\eta) \lambda \|\hbeta_{\GS^c}\|_{\G,1}  \leq 
  \|X(\bbeta-\obeta)\|_2^2 + O(|S| \lambda^2/n) .
 \]
 With optimal choice of $\lambda$, we have
  \[
  \|X(\hbeta-\obeta)\|_2^2 + (1-\eta) \lambda \|\hbeta_{\GS^c}\|_{\G,1}  \leq 
  \|X(\bbeta-\obeta)\|_2^2 + O(|S| m + \ln (q/|S|)) .
 \]

\subsubsection*{Tangent Space Analysis}

In this analysis, we assume that $\supp(\tilde{a}) \in \GS$.
We can then define
\[
  Q_G^\cT = \bbeta + \Delta Q, \quad
  \Delta Q_S = -0.5 (X_\GS^\top X_\GS)^{-1} (\lambda \sgn_\G(\bbeta_\GS)+\tilde{a}_\GS) \quad \text{ and } \quad \Delta Q_{\GS^c}=0 .
\]
We know that $Q_G^\cT$ is a dual certificate if 
\[
\|X_{\GS^c}^\top X_\GS (X_\GS^\top X_\GS)^{-1}\sgn(\bbeta_\GS) \|_{\G,\infty} \leq \eta-
\|X_{\GS^c}^\top X_\GS (X_\GS^\top X_\GS)^{-1}\tilde{a}_\GS -\tilde{b}_{\GS^c}\|_{\G,\infty}/\lambda . 
\]
This is essentially the irrepresentable condition of \cite{Bach08-groupLasso},
which reduces to the $\ell_1$ irrepresentable condition of \cite{ZhaoYu06} when $m=1$.
This condition implies the following oracle inequality:
\[
\|X(\obeta-\hbeta)\|_2^2 + (1-\eta) \lambda \|\hbeta_{S^c}\|_{\G,1} 
\leq \|X(\obeta-\bbeta)\|_2^2 + 
0.25 \lambda^2 \left\| (X_\GS^\top X_\GS)^{-1/2} (\sgn_\G(\bbeta_\GS) + \tilde{a}_\GS/\lambda)\right\|_2^2 .
\]
This oracle inequality generalizes a simpler result for $m=1$ in \cite{CandesPlan09}.

\subsubsection*{Simple Parameter Estimation Bounds}

Next, we consider the parameter estimation bound using Proposition~\ref{prop:param-est}. 
First, we consider the case of choosing $\tcT=\cT$; let
$\gamma_S$ be the smallest eigenvalue of $X_\GS^{\top}X_\GS$.
If we assume $\obeta \approx \bbeta$ and $\tilde{a}$ is small, we can expect a bound of the form:
\[
\|X(\obeta-\hbeta)\|_2^2 + (1-\eta) \lambda \|\hbeta_{\GS^c}\|_{\G,1} \leq \delta = O(\lambda^2 |S|/\gamma_S) ,
\]
where $\lambda=O(\sigma \sqrt{n (m+\ln q)})$.
Now, if we let $X_{\G_j}$ be the $j$-th group-column (with indices $\G_j$) of $X$, then
\[
\cor(\cT,\cT^\perp)
\leq \sup \left\{ \|(X_\GS^\top X_\GS)^{-1/2} X_\GS^\top X \beta_{\GS^c}\|_2 :
\lambda \|\beta_{\GS^c}\|_{\G,1} \leq \delta' \right\}
\leq \gamma_S^{-1/2} \max_{j \in S^c} \|X_\GS^\top X_{\G_j}\|_\spec \delta' /\lambda ,
\]
where 
\[
\gamma_S = \inf \, \{ \|X_\GS \beta_\GS \|_2^2 : \|\beta_\GS\|_2 =1\} 
\]
is the smallest eigenvalue of $X_\GS^\top X_\GS$.
Proposition~\ref{prop:param-est} gives 
$\|(\hbeta-\obeta)_{\GS^c}\|_{\G,1} \leq \delta'/\lambda$ and 
\[
\|(\hbeta-\obeta)_\GS\|_2
\leq \sqrt{\gamma_S^{-1} (1-\eta) \delta'} + 2 (\delta'/\lambda) \gamma_S^{-1} \max_{j \in S^c} \|X_S^\top X_{\G_j}\|_\spec ,
\]
where $\delta' = \delta/(1-\eta) + \lambda \|(\obeta)_{S^c}\|_1$,
and here we use $\|\cdot\|_\spec$ to denote the spectral norm of a matrix.

For the sake of illustration, we will next
assume that the standard error bound of $\delta' = O(\lambda^2 |S|/\gamma_S)$,
and the above result leads to the following bound
\[
\|(\hbeta-\obeta)_{\GS^c}\|_{\G,1} = O(\lambda |S|/\gamma_S), \quad
\|(\hbeta-\obeta)_\GS\|_2
\leq (\lambda \sqrt{|S|}/\gamma_S) \cdot 
O \left(1 + \sqrt{|S|} \gamma_S^{-1} \max_{j \in \GS^c} \|X_\GS^\top X_{\G_j}\|_\spec\right) . 
\]
If $X$ is very weakly correlated, $X_\GS^\top X_{\G_j}$ will be small.
In the ideal case $\gamma_S^{-1} \max_{j \in \GS^c} \|X_\GS^\top X_{\G_j}\|_\spec=O(1/\sqrt{|S|})$, we have
\[
\|\hbeta-\obeta\|_{\G,1} = O(\lambda |S|/\gamma_S) ,
\]
which is of the optimal order. However, in the pessimistic case of
$\gamma_S^{-1} \max_{j \in S^c} \|X_\GS^\top X_{\G_j}\|_\spec) =O(1)$, then we obtain
\[
\|\hbeta-\obeta\|_{\G,1} = O(\lambda |S|^{3/2}/\gamma_S) ,
\]
which has an extra factor of $\sqrt{|S|}$. 
Using the above derivation, the 2-norm error bound is always of the order 
\[
\|\hbeta-\obeta\|_2 \leq \|(\hbeta-\obeta)_\GS\|_2 + \|(\hbeta-\obeta)_{\GS^c}\|_{\G,1}
= O(\lambda |S|/\gamma_S) ,
\]
which has an extra factor of $\sqrt{|S|}$ compared to the ideal bound of
$\|\hbeta-\obeta\|_2 = O(\lambda \sqrt{|S|})$ in the earlier literature such as \cite{HuangZhang09,LMTG09}
under appropriately defined global restricted eigenvalue assumptions.

It should be mentioned that the assumptions we have made so far are relatively weak
without making global restricted eigenvalue assumptions, and thus the
resulting bound
$\|\hbeta-\obeta\|_2 = O(\lambda |S|)$ might be the best possible under these assumptions. 
In order to obtain the ideal bound of $\|\hbeta-\obeta\|_2 = O(\sqrt{|S|})$ (as appeared in the earlier literature),
we will consider adding extra assumptions.

\subsubsection*{Refined Parameter Estimation Bounds}
The first extra assumption we will make is that sparse eigenvalues are bounded from above, which prevents
the pessimistic case where $X_j$ are highly correlated for $j \in S^c$.
Such correlation can be defined with the upper sparse eigenvalue as:
\[
\rho^+(k) = \{\|X \beta\|_2^2/\|\beta\|_2^2 :  |\supp_\G(\beta)| \leq k\} ,
\]
where $\supp_\G(\beta) \subset \{1,\ldots,q\}$ is the (smallest) index set for groups of $\{\G_j\}$ that cover $\supp(\beta)$.
Using this notation, if
we choose the constrained $\Omega$ and $\obeta$ such that $\beta +\obeta \in \Omega$ implies that
$\|\beta\|_{\G,\infty} \leq M$ for some $M \leq \delta'/\lambda$, then 
it can be shown using the standard shifting argument for group $\ell_1$ regularization 
(e.g., \cite{HuangZhang09}) that for all positive integer $k \leq \delta'/(\lambda M)$:
\[
\cor(\cT,\cT^\perp) 
\leq \sup \left\{ \|X \beta_{\GS^c}\|_2 : \|\beta\|_{\G,\infty} \leq M , \;
\lambda \|\beta_{\GS^c}\|_{\G,1} \leq \delta' \right\}
\leq 2 \rho^+(k-1)^{1/2} \delta' /(\lambda \sqrt{k}) .
\]
This implies that
\[
\|(\hbeta-\obeta)_{\GS}\|_2 \leq
\sqrt{\gamma_S^{-1} (1-\eta) \delta'} + 4 \delta' \gamma_S^{-1/2} \rho^+(k-1)^{1/2}/(\lambda \sqrt{k}) ,
\]
and
\[
\|(\hbeta-\obeta)_{\GS^c}\|_2 \leq \sqrt{\delta' M/\lambda} \leq  \delta' /(\lambda \sqrt{k}) .
\]
Therefore assuming the standard error bound of $\delta' = O(\lambda^2 |S|/\gamma_S)$, we obtain
\[
\|\hbeta-\obeta\|_2 = O(\lambda \sqrt{|S|}/\gamma_S)
\inf_{k \leq \delta'/(\lambda M)} \left[ 1 + \sqrt{|S|/k} \sqrt{\rho^+(k)/\gamma_S} \right] .
\]
If $M$ is sufficiently small, then we can take $k$ sufficiently large so that $|S|=O(k)$, and it is possible to obtain error bound of
$\|\hbeta-\obeta\|_2 = O(\lambda \sqrt{|S|})$.

If we do not impose the $\|\cdot\|_{\G,\infty}$ norm constraint on $\Omega$, then another 
method is to choose $\tcT$ larger than $\cT$, which is the approach employed in 
\cite{CandesPlan11} for the standard $\ell_1$ regularization. Here we consider a similar assumption
for group-Lasso, where we define for all integer $k \geq 1$:
\[
\gamma_{S,k} = \inf \, \{ \|X \beta\|_2^2:  |\supp_\G(\beta)\setminus S| < k, \|\beta\|_2=1 \} .
\]
It is clear that $\gamma_S=\gamma_{S,1}$. Given any $k$ such that $\gamma_{S,k}$ is not too small,
we may define 
\[
\tcT=\{ \beta: \supp_\G(\beta) \subset \tilde{S} \} 
\]
and
\[
\tilde{S} =S \cup \{\text{group indices of largest $k-1$ absolute values of $\|\hbeta-\obeta\|_{G_j}: j \notin S$}\} .
\]
The smallest eigenvalue of $H_\tcT$ is no smaller than $\gamma_{S,k}$, and we also have
$\|(\hbeta-\obeta)_{\tilde{\GS}}\|_{\G,\infty} \leq M=\|(\hbeta-\obeta)_{\GS^c}\|_{\G,1}/k \leq \delta'/(k \lambda)$.
Using the same derivation as before, we have
\[
\|\hbeta-\obeta\|_2 = O(\lambda \sqrt{|S|}/\gamma_{S,k})
\left[ 1 + \sqrt{|S|/k} \sqrt{\rho^+(k)/\gamma_{S,k}} \right] .
\]
This means that if we can choose $k$ at the order of $|S|$ such that 
$\rho^+(k)/\gamma_{S,k}=O(1)$, then we have
\[
\|\hbeta-\obeta\|_2 = O(\lambda \sqrt{|S|}) .
\]
In the standard $\ell_1$ case, the requirement of $\rho^+(k)/\gamma_{S,k}=O(1)$ is also needed in the so-called
``RIP-less'' approach of \cite{CandesPlan11} to obtain the ideal bound for
$\|\hbeta-\obeta\|_2$. The approach is called ``RIP-less'' because this condition is weaker than the classical RIP condition 
of \cite{CandesTao07} (or its group-Lasso counterpart in \cite{HuangZhang09}) that is far more restrictive.
This bound is also flexible as we can choose any $k\geq 1$: in the worst case of $k=1$, we have 
$\|\hbeta-\obeta\|_2 = O(\lambda |S|)$ with an extra $\sqrt{|S|}$ factor. 
This extra factor can be removed as long as we take $k$ at the order of $|S|$.

\subsection{Matrix completion}

Let $\Omegabar$ be the set of $p \times q$ matrices, and assume that the inner product is defined as
$\innerprod{\beta}{\beta'}= \tr(\beta^\top \beta')$.

We consider $x_1,\ldots,x_n$ and observe 
\[
y_i = \innerprod{x_i}{\bbeta_*}+ \epsilon_i ,
\]
where $\{\epsilon_i\}$ are noises.
In order to recover $\bbeta_*$, we consider the following convex optimization problem:
\[
\hbeta = \arg\min\left[ \sum_{i=1}^n (\innerprod{x_i}{\beta} - y_i)^2 + \lambda \|\beta\|_* \right] ,
\]
where $\|\beta\|_*$ is the trace-norm of matrix $\beta$, defined as the sum of its singular values.

In the following, we will briefly discuss results that can be obtained from our analysis using the tangent space analysis.
For simplicity, we will keep the discussion at a relatively high level, with some detailed discussions skipped.

We assume that $\bbeta$ is of rank-$r$, and
$\bbeta=U \Sigma V^\top$ is the SVD of $\bbeta$, where $U$ and $V$ are $p \times r$
and $q \times r$ matrices.
The tangent space is defined as $\cT=\{\beta: P_\cT(\beta) = \beta\}$, 
where $P_\cT(\beta) = U U^\top \beta + \beta V V^\top - U U^\top \beta V V^\top$.

Using notations in Section~\ref{sec:struct-L1}, we may take $e_S(W)=U V^\top$ and
$e_{S^c}(\scrB)=\{b \in \cT^\perp: \|b\|_\spec \leq \lambda\}$ in (\ref{eq:structG}).
Therefore
\[
\|\beta\|_\scrB = \lambda \|P_\cT^\perp \beta\|_* \qquad
\|P_\cT^\perp b\|_\scrBD =  \| P_\cT^\perp b\|_\spec/\lambda .
\]
This means that we may take $G$ in (\ref{eq:structG-int}) as
$G=\{u: \; P_\cT u=\lambda U V^\top \; \& \; \|P_\cT^\perp u\|_\spec \leq \eta \lambda\}$ for some $0 \leq \eta \leq 1$,
which implies that $R(\beta)-R_G(\beta) \geq (1-\eta) \|P_\cT^\perp \beta\|_1$. 

We further consider target $\obeta$ that satisfies (\ref{eq:target}), which we can rewrite as
\[
2 \sum_{i=1}^n x_i (\innerprod{x_i}{\obeta -\bbeta_*}- \epsilon_i) = \tilde{a} + \tilde{b} ,
\]
where $\tilde{b} \subset \cT^\perp$, and $\|\tilde{b}\| = \tilde{\eta} \lambda$.
We assume that $\|\tilde{a}\|_2$ is small.

For matrix completion, we assume that $\{x_i\}$ are matrices of the form $e_{a,b}$ with
1 at entry $(a,b)$ and 0 elsewhere, where $(a,b)$ is uniformly at random.
It can be shown  using techniques of \cite{CandesR09,Recht09}
that under appropriate incoherence conditions, 
a tangent space dual certificate can be constructed with large probability that satisfies 
(\ref{eq:irrep-struct}). Due to the space limitation, we skip the details.
This leads to
\[
D_L(\obeta,\hbeta)+ (1-\eta)  \lambda \|P_\cT^\perp \hbeta \|_*
\leq D_L(\obeta,\bbeta) + \delta,
\quad \delta= 0.25 \innerprod{\lambda U V^\top +\tilde{a}}{H_\cT^{-1} (\lambda U V^\top +\tilde{a})} .
\]

Note that for sufficiently large $n$, the smallest eigenvalue of $H_\cT$ can be lower bounded as $O(pq/n)$.
Since $\innerprod{\lambda UV^\top}{\lambda UV^\top}= \lambda^2 r$,
we may generally choose $\lambda$ such that $\innerprod{\tilde{a}}{\tilde{a}}= O(\lambda^2 r)$, we thus obtain
the following oracle inequality for matrix completion:
\[
D_L(\obeta,\hbeta)+ (1-\eta)  \lambda \|P_\cT^\perp \hbeta \|_*
\leq D_L(\obeta,\bbeta) + O(\lambda^2 p q r/n) .
\]
If $\epsilon_i$ are iid Gaussian noise $N(0,\sigma^2)$, then we may choose
$\lambda$ at the order $\sigma \sqrt{n \ln \max(p,q)/\min(p,q)}$. This gives
\[
D_L(\obeta,\hbeta)+ (1-\eta)  \lambda \|P_\cT^\perp \hbeta \|_*
\leq D_L(\obeta,\bbeta) + O(\sigma^2 \max(p,q) r \ln (p+q)) .
\]

In the noise-free case, we can let $\lambda \to 0$, and exact recovery is obtained.
This complements a related result of \cite{KoTsLo10} that does not lead to exact recovery even when $\sigma=0$.
In the noisy case, parameter estimation bounds can be obtained in a manner analogous to
the parameter estimation bound for group $\ell_1$ regularization. 
Due to the space limitation, we will leave the details to a dedicated report.

\subsection{Mixed norm regularization}
\label{sec:mixed-norm}
The purpose of this example is to show that the dual certificate analysis can be applied to more complex
regularizers that may be difficult to analyze using traditional ideas such as the RIP analysis.
The analysis is similar to that of group $\ell_1$ regularization but with more complex
calculations. For simplicity, we will only provide a sketch of the analysis while skipping some of the details.

We still consider the regression problem 
\[
y= X \bbeta_* + \epsilon ,
\]
where for simplicity we only consider Gaussian noise $\epsilon \sim N(0,\sigma^2 I_{n \times n})$. 
We assume that $p=q m$, and the variables $\{1,\ldots,p\}$ are divided into $q$ non-overlapping blocks
$\G_1,\ldots,\G_{q} \subset \{1,\ldots,p\}$, each block of size $m$. 

The standard sparse regularization methods are either using the Lasso regularizer of (\ref{eq:L1}) or 
using the group-Lasso regularizer of (\ref{eq:group-L1}). 
Let $S_S=\supp(\bbeta)$ and $S_\G=\supp_\G(\bbeta)$, we know that 
under suitable restricted strong convexity conditions, the following oracle inequality holds for 
the Lasso regularizer (\ref{eq:L1}) 
\[
\|X(\obeta-\hbeta)\|_2^2+ (1-\eta) \lambda \|\hbeta_{S_S^c}\|_{\G,1} 
\leq \|X(\obeta-\bbeta)\|_2^2 +  O(\sigma^2 n |S_S| \ln p /\gamma_{S_s}) ,
\]
and the following oracle inequality holds for the group Lasso regularizer (\ref{eq:group-L1}):
\[
\|X(\obeta-\hbeta)\|_2^2 + (1-\eta) \lambda \|\hbeta_{\GS_\G^c}\|_{\G,1} 
\leq \|X(\obeta-\bbeta)\|_2^2 + O(\sigma^2 n |S_\G| (m + \ln q) /\gamma_{S_\G}) .
\]
Note that we always have $|S_S| \leq |S_\G| m$. 
By comparing the above two oracle inequalities, we can see that the benefit of using group sparsity is when 
$|S_S| \approx |S_\G| m$, which means that sparsity pattern occur in groups,
and the group structure is correct. In such case, the dimension dependency reduces from
$|S_S| \ln p$ to $|S_\G| \ln q \approx m^{-1} |S_S| \ln q$. However,
if some of the signals do not occur in groups, then it is possible that $|S_\G| m$ can be much larger than $|S_S|$,
and in such case, Lasso is superior to group Lasso.

It is natural to ask whether it is possible to combine the benefits of Lasso and group Lasso regularizers.
Assume that $\bbeta$ is decomposed into two parts $\bbeta=\tbeta'+\tbeta''$ so that 
$\tbeta''$ covers nonzeros of $\tbeta$ that occur in groups, and $\tbeta'$ covers nonzeros of $\tbeta$ that 
do not occur in groups. Ideally we would like to achieve an oracle inequality of
\begin{align}
&\|X(\obeta-\hbeta)\|_2^2 + (1-\eta) \lambda \|\hbeta\| \label{eq:opt-decomp}
\\
\leq& \|X(\obeta-\bbeta)\|_2^2 + 
O\left(\frac{\sigma^2 n}{\gamma} \left(|\supp(\tbeta')| \ln p + |\supp_\G(\tbeta'')|(m + \ln q)\right) \right)  \nonumber \\
=& \|X(\obeta-\bbeta)\|_2^2 + 
O\left(\frac{\sigma^2 n}{\gamma} \left(|\supp(\bbeta \setminus \cup_{j \in \tilde{S}} \G_j) | \ln p + |\tilde{S}|(m + \ln q)\right) \right) , \nonumber
\end{align}
where $\|\hbeta\|$ is a certain seminorm of $\hbeta$, and $\tilde{S}=\{j: (m + \ln q) \leq c |\supp(\bbeta_{\G_j})|\ln p\}$
for some constant $c>0$. 
We note that the optimal decomposition can be achieved by taking
$\tbeta'_{\G_j}=0$ with $\tbeta''_{\G_j}=\bbeta_{\G_j}$ when $j \in S'$ and
and $\tbeta''_{\G_j}=0$ with $\tbeta'_{\G_j}=\bbeta_{\G_j}$ otherwise.

In the following, we show that the oracle inequality of (\ref{eq:opt-decomp}) can be achieved via a mixed norm regularizer 
defined below:
\begin{equation}
R(\beta) =\inf_{\beta=\beta'+\beta''} \left[ \lambda_1 \|\beta'\|_1 + \lambda_\G \|\beta''\|_{\G,1} \right] .
\label{eq:reg-mixed}
\end{equation}
This mixed regularizer can be referred to  as the infimal convolution of
Lasso and group Lasso regularizers, and it is a special case of \cite{Jacob09icml}.
If we can prove an oracle inequality of (\ref{eq:opt-decomp}) for this regularizer, then it 
means that we can adaptively 
decompose the signal $\bbeta$ into two parts $\beta'$ and $\beta''$ in order to achieve the most significant benefits
with standard sparsity bound for $\beta'$ and group sparsity bound for $\beta''$ (without knowing the decomposition a priori).

We will consider the decomposed parametrization $[\beta',\beta'']$, and the mixed norm regularizer 
(\ref{eq:reg-mixed}) becomes a special case of (\ref{eq:struct-L1}). 
Although the loss function $L(\cdot)$ is not strongly convex with respect to this parametrization,
this does not cause problems because we are only interested
in $\beta=\beta'+\beta''$. Since $L(\cdot)$ is strongly convex with respect to $\beta$
with an appropriate tangent space $\cT$, we only need to consider the direction along $\beta=\beta'+\beta''$ when applying
the results. In this regard, it is easy to verify that at the optimal decomposition in (\ref{eq:reg-mixed}),
there exist $u' \in \partial \|\beta'\|_1$ and $u'' \in \partial \|\beta''\|_{\G,1}$
such that $\lambda_1 u' = \lambda_\G u''$. Moreover, for any such $(u',u'')$, $\lambda_1 u' \in \partial R(\beta)$.

In order to define $\cT$, we first define $\scrB$. Consider
$S_\G=\{j: \lambda_\G < 2 \lambda_1 \|\sgn(\bbeta)_{\G_j}\|_2\}$, with the corresponding
support $\GS_\G=\cup_{j \in S_\G} \G_j$.
The meaning of $S_\G$ is that groups in $S_\G$ are allowed to use both standard and group sparsity to represent $\bbeta$, 
while groups in $S_\G^c$ always use standard sparsity only.
The set $\GS_\G$ will expand the tangent space for the nonzero group sparsity elements.
We also define the tangent space support set for single sparsity elements as 
$\GS_1= \supp(\bbeta) \cup \GS_\G$.
Let 
\begin{equation}
[\bbeta',\bbeta'']=\arg\min_{(\beta',\beta''): \bbeta=\beta'+\beta''}\left[ \lambda_1 \|\beta'\|_1 + \lambda_\G \|\beta''\|_{\G,1} 
\right] .
\label{eq:mixed-norm-optdecomp}
\end{equation}
It satisfies $\lambda_1 \nabla \|\bbeta'\|_1 =\lambda_\G \nabla \|\bbeta''\|_{\G,1}$, and
$\nabla R(\bbeta)=\lambda_1 \nabla \|\bbeta'_{\G_j}\|_1 +\lambda_\G \nabla \|\bbeta''_{\G_j}\|_{\G,1}$.
Consider $\G_j$ such that $\bbeta''_{\G_j} \neq 0$, we obtain from 
$\lambda_1 \nabla \|\bbeta'\|_1 =\lambda_\G \nabla \|\bbeta''\|_{\G,1}$ that
$[\nabla \|\bbeta'_{\G_j}\|_{1}]_i \neq 0$ only when $\bbeta_i \neq 0$ for $i \in \G_j$; therefore
$\|(\nabla \|\bbeta'_{\G_j}\|_{1})\|_2 \leq \|\sgn(\bbeta)_{\G_j}\|_2$, and thus
$\lambda_\G \leq \lambda_1 \|(\nabla \|\bbeta'_{\G_j}\|_{1})\|_2 \leq \lambda_1 \|\sgn(\bbeta)_{\G_j}\|_2$.
It implies that $j \in S_\G$ and thus $\supp(\bbeta'') \subset S_\G$.
Now we can define $W$ and $\scrB$ as
\[
e_S(W) = \lambda_1 \nabla \|\bbeta'\|_1=\lambda_\G \nabla \|\bbeta''\|_{\G,1}
\]
where we can take $[\nabla \|\bbeta'\|_1]_j=0$ when $j \notin \GS_1$; and define
\[
e_{S^c}(\scrB)=\{u_{\GS_1^c} : \|u_{\GS_1^c}\|_\infty \leq \lambda_1 \; \& \; \|u_{\GS_\G^c}\|_{\G,\infty} \leq 0.5 \lambda_\G \} .
\]
With the above choices, we have for all $u \in e_{S^c}(\scrB)$, $e_S(W) + u \in \partial R(\bbeta)$ because
it can be readily checked that $e_S(W) + u \in \partial (\lambda_1 \|\bbeta'\|_1) \cap \partial (\lambda_\G \|\bbeta''\|_{\G,1})$.
Moreover, we have
\[
\sup_{u \in e_{S^c}(\scrB)} \innerprod{u}{\beta} = 
\min_{\hbeta_{\GS_1^c}=\beta'_{\GS_1^c}+\beta''_{\GS_\G^c}}
\left[\lambda_1 \|\beta'_{\GS_1^c}\|_1 + 0.5 \lambda_\G \|\beta''_{\GS_\G^c}\|_{\G,1} \right] .
\]
We can thus define $G$ according to (\ref{eq:structG-int}) as
\[
G=\{e_S(W)+\eta u : u \in e_{S^c}(\scrB)\} ,
\]
so that $G \subset \partial R(\bbeta)$. 

For simplicity, we assume that $\obeta$ satisfies (\ref{eq:target}) with $\tilde{a}=0$, which can be achieved with the construction in (\ref{eq:target-opt}).
With these choices, we obtain from  (\ref{eq:recovery-global-dc-quadratic-struct}) the following
oracle inequality (under appropriate restricted eigenvalue condition with parameter $\gamma$):
\begin{align*}
& \| X (\hbeta-\obeta)\|_2^2+ (1-\eta) 
\min_{\hbeta_{\GS_1^c}=\beta'_{\GS_1^c}+\beta''_{\GS_\G^c}}
\left[\lambda_1 \|\beta'_{\GS_1^c}\|_1 + 0.5 \lambda_\G \|\beta''_{\GS_\G^c}\|_{\G,1} \right]\\
\leq &\| X (\bbeta-\obeta)\|_2^2 + \gamma^{-1}  O(\|e_S(W)\|_2^2) \\
\leq& \| X (\bbeta-\obeta)\|_2^2 + \gamma^{-1}
O \left( \lambda_1^2 |\supp(\bbeta)\setminus S_\G | + \lambda_\G^2 |S_\G| \right) .
\end{align*}
The last inequality follows from 
\[
\|e_S(W)\|_2^2 \leq \sum_{j \in S_\G} \lambda_\G^2 \|(\nabla \|\bbeta''\|_{\G,1})_{\G_j}\|_2^2 +
\sum_{j \notin S_\G} \lambda_1^2 \|(\nabla \|\bbeta'\|_{1})_{\G_j}\|_2^2 ,
\]
which is a consequence of $e_S(W)=\lambda_1 \nabla \|\bbeta'\|_1 =\lambda_\G \nabla \|\bbeta''\|_{\G,1}$.

Similar to the standard Lasso and group Lasso cases, for mixed norm regularization, we may still
choose the Lasso regularizer parameter $\lambda_1=c_1 \sigma \sqrt{n \ln(p)}$,
and the group Lasso regularization parameter $\lambda_\G= c_2 \sigma \sqrt{n (m+\ln (q))}$
so that (\ref{eq:target}) holds ($c_1,c_2>0$ are constants). 
Plug in these values, we obtain the following oracle inequality with this choice of parameters:
\begin{align*}
& \| X (\hbeta-\obeta)\|_2^2+ (1-\eta) 
\min_{\hbeta_{\GS_1^c}=\beta'_{\GS_1^c}+\beta''_{\GS_\G^c}}
\left[\lambda_1 \|\beta'_{\GS_1^c}\|_1 + 0.5 \lambda_\G \|\beta''_{\GS_\G^c}\|_{\G,1} \right]\\
\leq & \| X (\bbeta-\obeta)\|_2^2 + \gamma^{-1} n \sigma^2 \cdot 
O \left( (n+\ln p) |\supp(\bbeta)\setminus \GS_\G| + |S_\G|(m + \ln q) \right) .
\end{align*}
Since the definition of $S_\G$ is such that $j \in S_\G$ when
$m+\ln (q) \leq 4 (c_1/c_2)^2 |\supp(\bbeta_{\G_j})| \ln(p)$,
the right hand side achieves the optimal decomposition error 
bound in (\ref{eq:opt-decomp}). This means that the mixed norm regularizer (\ref{eq:reg-mixed})
achieves optimal adaptive decomposition of standard and group sparsity. 

\subsection{Generalized linear models} 
\label{sec:genlin-example}
Results for generalized linear models can be easily obtained under the general framework of this paper, as
discussed after Corollary~\ref{cor:dual_certificate-oracle}. 
This section presents a more elaborated treatment. 
In generalized linear models, we may write the negative log likelihood as 
\begin{equation}
L(\beta) = \sum_{i=1}^n\ell_i(\innerprod{x_i}{\beta}) , \label{eq:gen-lin-model}
\end{equation}
where $x_i\in\Omegabar^*$ and $\ell_i$ may depend on certain response variable $y_i$. 
Suppose $\ell_i(t)$ are convex and twice differentiable. Let  
\bes
{\kappa} = \max_{i\le n}\sup_{s<t}\big|\log(\ell_i''(t))-\log(\ell_i''(s))\big|\big/|t-s|
\ees
be the maximum Lipschitz norm of $\log(\ell_i''(t))$. We note that $\kappa=1$ for logistic 
regression with $\ell_i(t) = \ln(1+e^{-t})$, $\kappa=1$ for the Poisson/log linear regression 
with $\ell_i(t) = e^t - y_i t$, and $\kappa = 0$ for linear regression. 
For sparse $\bbeta$, $\cC\subset\Omega$, norm $\|\cdot\|$, and $j=1,2$, define 
\bes
\gamma_j(\bbeta;r,\cC,\|\cdot\|) = \inf\Big\{\sum_{i=1}^n\frac{\ell_i''(\innerprod{x_i}{\bbeta})}{2e}
\min\Big(\frac{\innerprod{x_i}{\beta-\bbeta}^2}{\|\beta-\bbeta\|^2},
\frac{|\innerprod{x_i}{\beta-\bbeta}|^{2-j}}{r^j\|\beta-\bbeta\|^{2-j}}\Big): 
\beta\in\cC\Big\}. 
\ees
The following lemma can be used to bound $D_L(\beta,\bbeta)$ and $D_L(\bbeta,\beta)$ 
from below. 

\begin{lemma}\label{lm:GLM} 
Given $\bbeta$, $\cC\subset\Omega$ and norm $\|\cdot\|$, 
let $\beta\in\Omegabar$ such that the ray from $\bbeta$ to $\beta$ and beyond intersects with $\cC$, 
$\{t(\beta-\bbeta)+\bbeta: t>0\}\cap\cC\neq\emptyset$. 
If $0<\|\beta-\bbeta\|\le r$, 
\bes
\frac{D_L(\beta,\bbeta)}{\|\beta-\bbeta\|^2} \ge 
\gamma_1(\bbeta; \kappa r,\cC,\|\cdot\|),\quad 
\frac{D_L(\bbeta,\beta)}{\|\beta-\bbeta\|^2} \ge 
\gamma_2(\bbeta; \kappa r,\cC,\|\cdot\|). 
\ees
\end{lemma}

\begin{proof} Let $\tbeta = t_0(\beta - \bbeta) + \bbeta \in\cC$. 
Since $\gamma_j(\bbeta;r,\cC,\|\cdot\|)$ is decreasing in $r$, it suffices to 
consider $0< \|\beta-\bbeta\| =r$. 
Since ${\kappa}$ is the Lipschitz norm of $\log(\ell_i''(t))$, 
\bes
D_L(\beta,\bbeta)/r^2 
&=& \int_0^1(1-t) \sum_{i=1}^n\ell_i''(\innerprod{x_i}{\bbeta}+t\innerprod{x_i}{\beta-\bbeta})
\innerprod{x_i}{\beta-\bbeta}^2 dt/r^2
\cr &\ge& \int_0^1(1-t) \sum_{i=1}^n\ell_i''(\innerprod{x_i}{\bbeta})
e^{-t{\kappa}|\innerprod{x_i}{\beta-\bbeta}|}\innerprod{x_i}{\beta-\bbeta}^2 dt/r^2
\cr &\ge& \sum_{i=1}^n\ell_i''(\innerprod{x_i}{\bbeta})\innerprod{x_i}{\beta-\bbeta}^2
\int_0^1(1-t)I\{t{\kappa}|\innerprod{x_i}{\beta-\bbeta}|\le 1\}dt/(er^2). 
\ees
Since $\int_0^{x\wedge 1}(1-t)dt = x\wedge 1 - (x\wedge 1)^2/2\ge (x\wedge 1)/2$, we find 
\bes
D_L(\beta,\bbeta)/r^2 
&\ge& \sum_{i=1}^n\ell_i''(\innerprod{x_i}{\bbeta})\innerprod{x_i}{\beta-\bbeta}^2
\min\Big(1, \frac{1}{\kappa |\innerprod{x_i}{\beta-\bbeta}|}\Big)\frac{1}{2er^2}
\cr &\ge& \sum_{i=1}^n\frac{\ell_i''(\innerprod{x_i}{\bbeta})}{2e}
\min\Big(\frac{\innerprod{x_i}{\beta-\bbeta}^2}{\|\beta-\bbeta\|^2},
\frac{|\innerprod{x_i}{\beta-\bbeta}|}{\kappa r\|\beta-\bbeta\|}\Big). 
\ees
Since $\innerprod{x_i}{\beta-\bbeta}|/\|\beta-\bbeta\|=\innerprod{x_i}{\tbeta-\bbeta}|/\|\tbeta-\bbeta\|$ 
and $\tbeta\in \cC$, $D_L(\beta,\bbeta)/r^2\ge \gamma_1(\bbeta; \kappa r,\cC,\|\cdot\|)$. 

The proof for $D_L(\bbeta,\beta)$ is similar. We have 
\bes
D_L(\bbeta,\beta)/r^2 
&=& \int_0^1 \sum_{i=1}^n\ell_i''(\innerprod{x_i}{\bbeta}+t \innerprod{x_i}{\beta-\bbeta})\innerprod{x_i}{\beta-\bbeta}^2 tdt /r^2
\cr &\ge & \sum_{i=1}^n\ell_i''(\innerprod{x_i}{\bbeta})\innerprod{x_i}{\beta-\bbeta}^2
\int_0^1I\{t{\kappa} |\innerprod{x_i}{\beta-\bbeta}| \le 1\} tdt/(er^2)
\cr &=& (2e)^{-1}\sum_{i=1}^n \ell_i''(\innerprod{x_i}{\bbeta})
\min\Big(\frac{\innerprod{x_i}{\beta-\bbeta}^2}{\|\beta-\bbeta\|^2}, \frac{1}{{\kappa}^2r^2}\Big). 
\ees
This gives $D_L(\bbeta,\beta)/r^2\ge \gamma_2(\bbeta; \kappa r,\cC,\|\cdot\|)$. 
\end{proof}

Suppose $\gamma_j(\bbeta;r_0,\cC,\|\cdot\|)\ge \gamma_0$ for $j=1,2$. 
Lemma \ref{lm:GLM} asserts that for the $\beta$ considered, both $D_L(\beta,\bbeta)$ and 
$D_L(\bbeta,\beta)$ are no smaller than $\gamma_0\|\beta-\bbeta\|^2$ 
for ${\kappa}\|\beta-\bbeta\|\le r_0$. For larger $r=\|\beta-\bbeta\|$, 
\bes
&& D_L(\beta,\bbeta)\ge r^2\gamma_1(\bbeta; \kappa r,\cC,\|\cdot\|) 
\ge \|\beta-\bbeta\| (r_0/{\kappa})\gamma_1(\bbeta; r_0,\cC,\|\cdot\|),\qquad 
\cr && D_L(\bbeta,\beta)\ge r^2 \gamma_2(\bbeta; \kappa r,\cC,\|\cdot\|)
\ge (r_0/{\kappa})^2\gamma_2(\bbeta; r_0,\cC,\|\cdot\|). 
\ees
Since $D_L(\beta,\bbeta)$ is convex in $\bbeta$ and $D_L(\bbeta,\beta)$ is not, such 
lower bounds are of the best possible type for large $\|\beta-\bbeta\|$ 
when $\ell_i''(t)$ are small for large $t$, as in the case of logistic regression. 

Given $\bbeta$, setting $\cC_G=\{\beta: \sup_{u\in G} \innerprod{u+\nabla L(\bbeta)}{\beta} \le 0\}$ 
yields the lower bound 
\bes
\gamma_L(\bbeta;r,G,\|\cdot\|) \ge \gamma_1(\bbeta;\kappa r,\cC_G,\|\cdot\|)
\ees 
for the RSC constant in Definition \ref{def:RSC}. The lower bound 
$D_L(\bbeta,\beta)\ge (r_0/{\kappa})^2\gamma_2(\bbeta; r_0,\cC,\|\cdot\|)$ can be 
used to check the condition $D_L(\bbeta,\beta)\ge D_{\bar{L}}(\beta,\bbeta)$ 
in Corollaries \ref{cor:dual_certificate-oracle} and \ref{cor:recovery-global-dc-oracle}. 

We measure the noise level by 
\bes
\eta(\obeta) = \sup\big\{|\innerprod{\nabla L(\obeta)}{\beta}|/R(\beta): \beta\neq 0,\beta\in\Omega\big\}. 
\ees
Let $\bbeta$ be a sparse vector and $G\subseteq\pa R(\bbeta)$. 
Given $\{\bbeta,\obeta,\|\cdot\|\}$, we measure the penalty level by 
\bes
\lam(\bbeta,\obeta;\|\cdot\|)
= \sup\big\{ \innerprod{\nabla L(\obeta)+u}{\bbeta-\beta}/\|\beta-\bbeta\|: 
u\in\pa R(\beta),\beta\in\Omega\}. 
\ees 
Since for all $u \in \pa R(\beta)$ and $\bar{u} \in \pa R(\bbeta)$, we have
$\innerprod{u}{\bbeta-\beta} \leq \innerprod{\bar{u}}{\bbeta-\beta}$, it follows that
\[
\lam(\bbeta,\obeta;\|\cdot\|)
\leq \inf_{\bar{u} \in \pa R(\bbeta)} \|\nabla L(\obeta)+\bar{u}\|_D ,
\]
where $\|\cdot\|_D$ is the dual norm of $\|\cdot\|$.
This connects the quantity $\lam(\cdot)$ to 
$\inf_{u  \in G} \|u+\nabla L(\bbeta)\|_D$ used in Theorem~\ref{thm:dual_certificate-error}.

Similarly, we may define 
\bes
\cC_{\bbeta,\obeta} = \Big\{\beta: \sup_{u\in\pa R(\beta)}\innerprod{u+\nabla L(\obeta)}{\bbeta-\beta} > 0\Big\} .
\ees
Note that  we have
$\cC_{\bbeta,\obeta} \subset \Big\{\beta: \inf_{\bar{u}\in\pa R(\bbeta)}\innerprod{\bar{u}+\nabla L(\obeta)}{\bbeta-\beta} > 0\Big\}$,
and this relationship connects the quantity 
$\gamma_2(\bbeta;1,\cC_{\bbeta,\obeta}\|\cdot\|)$ in Theorem~\ref{thm:GLM} to the quantity
$\gamma_L(\bbeta;r,G,\|\cdot\|)$ in Defintion~\ref{def:RSC}.
The following result  for generalized linear models is related to Theorem~\ref{thm:dual_certificate-error}, but 
is more specific to the loss function (\ref{eq:gen-lin-model}) and more elaborated. 

\begin{theorem}\label{thm:GLM} Suppose $\eta(\obeta)<1$. Let $\bbeta$ be a sparse vector such that 
\bel{thm:GLM-1}
\sup_{\beta\in\cC_{\bbeta,\obeta}}\|\beta-\bbeta\|
\le \frac{\gamma_2(\bbeta;1,\cC_{\bbeta,\obeta}\|\cdot\|)}
{{\kappa}^2\lam(\bbeta,\obeta;\|\cdot\|)}
+ \frac{\lam(\bbeta,\obeta;\|\cdot\|)}
{4\gamma_2(\bbeta;1,\cC_{\bbeta,\obeta}\|\cdot\|)}. 
\eel
Then, 
\bes
D_L(\hbeta,\obeta) \le D_L(\bbeta,\obeta) + \frac{\lam^2(\bbeta,\obeta;\|\cdot\|)}
{4\gamma_2(\bbeta;1,\cC_{\bbeta,\obeta}\|\cdot\|)}. 
\ees
\end{theorem}

\begin{proof}. 
Let $\tbeta = \bbeta+t(\hbeta-\bbeta)$. Define 
\bes
f(t) = D_L(\tbeta,\obeta) - D_L(\bbeta,\obeta) 
= L(\tbeta)-L(\bbeta) + t\innerprod{- \nabla L(\obeta)}{\hbeta-\bbeta}. 
\ees
The function $f(t)$ is convex with $f(0)=0$ and 
$f'(t) = \innerprod{\nabla L(\tbeta) - \nabla L(\obeta)}{\hbeta-\bbeta}$. 
If $f'(1)\le 0$, then $D_L(\hbeta,\obeta) - D_L(\bbeta,\obeta) = f(1)\le f(0)=0$ 
and the conclusion holds.  Assume $f'(1)>0$ in the sequel. 

Let $u=-\nabla L(\hbeta)$. By (\ref{eq:hbeta}), $u\in\pa R(\hbeta)$. 
Since $f'(1)=\innerprod{u+ \nabla L(\obeta)}{\bbeta-\hbeta}>0$, $\hbeta\in \cC_{\bbeta,\obeta}$. 
It follows that $f'(1)\le \lam(\bbeta,\obeta;\|\cdot\|)\|\hbeta-\bbeta\|$. By Lemma \ref{lm:GLM}
\bes
f(t)-f'(t)t= - D_L(\bbeta,\tbeta) 
\le - \|\tbeta-\bbeta\|^2\gamma_2(\bbeta;\kappa\|\tbeta-\bbeta\|,\cC_{\bbeta,\obeta}\|\cdot\|).
\ees 
Consider two cases. If $\kappa \|\hbeta-\bbeta\|\le 1$, we set $t=1$ to obtain 
\bes
f(1) &\le& f'(1) - \|\hbeta-\bbeta\|^2\gamma_2(\bbeta;1,\cC_{\bbeta,\obeta}\|\cdot\|)
\cr &\le& \lam(\bbeta,\obeta;\|\cdot\|)\|\hbeta-\bbeta\| - 
\|\hbeta-\bbeta\|^2\gamma_2(\bbeta;1,\cC_{\bbeta,\obeta}\|\cdot\|). 
\ees
Taking the maximum of $x\lam(\bbeta,\obeta;\|\cdot\|) - 
x^2\gamma_2(\bbeta;\kappa\|\hbeta-\bbeta\|,\cC_{\bbeta,\obeta}\|\cdot\|)$, we find 
\bes
D_L(\hbeta,\obeta) - D_L(\bbeta,\obeta) = f(1) 
\le \frac{\lam^2(\bbeta,\obeta;\|\cdot\|)}
{4\gamma_2(\bbeta;1,\cC_{\bbeta,\obeta}\|\cdot\|)}. 
\ees
For $\kappa \|\hbeta-\bbeta\| > 1$, we set $t<1$ so that $\kappa \|\tbeta-\bbeta\| = 1$
\bes
f(1) \le f'(1) + f(t)-tf'(t) \le \lam(\bbeta,\obeta;\|\cdot\|)\|\tbeta-\bbeta\| 
- {\kappa}^{-2}\gamma_2(\bbeta;1,\cC_{\bbeta,\obeta}\|\cdot\|).
\ees
This gives $f(1)\le  \lam^2(\bbeta,\obeta;\|\cdot\|)/
\{4\gamma_2(\bbeta;1,\cC_{\bbeta,\obeta}\|\cdot\|)\}$ when 
\bes
\|\tbeta-\bbeta\| \le \frac{\gamma_2(\bbeta;1,\cC_{\bbeta,\obeta}\|\cdot\|)}
{{\kappa}^2\lam(\bbeta,\obeta;\|\cdot\|)}
+ \frac{\lam(\bbeta,\obeta;\|\cdot\|)}
{4\gamma_2(\bbeta;1,\cC_{\bbeta,\obeta}\|\cdot\|)}. 
\ees
The proof is complete in view of the assumed condition on $\bbeta$. 
\end{proof}

Condition (\ref{thm:GLM-1}) holds if $\sup_{\beta\in\Omega}\|\beta\|\le A$ and 
$2A \le \gamma_2(\bbeta;1,\cC_{\bbeta,\obeta}\|\cdot\|)/
\{{\kappa}^2\lam(\bbeta,\obeta;\|\cdot\|)\}$. This is a weaker condition that 
the condition discussed after Corollary~\ref{cor:dual_certificate-oracle} because the quantity
$\lam(\bbeta,\obeta;\|\cdot\|) \leq \inf_{\bar{u} \in \pa R(\bbeta)} \|\nabla L(\obeta)+\bar{u}\|_D $ is generally very small, which
means that we allow a very large $A$. Under this relatively weak condition, Theorem~\ref{thm:GLM} gives an oracle
inequality for generalized linear models that can be easily applied to common formulations such as logistic regression
and Poisson regression.

\bibliographystyle{abbrv}
\bibliography{dual_certificate}

\end{document}